\documentclass[pmlr]{jmlr}
\usepackage{amsmath, amssymb}
\usepackage{graphicx}
\usepackage{float}
\usepackage{hyperref}
\hypersetup{
  colorlinks=true,
  linkcolor=blue,
  citecolor=blue,
  urlcolor=blue,
  pdftitle={CRPS-Optimal Binning for Univariate Conformal Regression},
  pdfauthor={Paolo Toccaceli},
}
\usepackage{tikz}
\usepackage{caption}
\jmlrvolume{329}
\jmlryear{2026}
\jmlrworkshop{Conformal and Probabilistic Prediction with Applications}
\jmlrproceedings{PMLR}{Proceedings of Machine Learning Research}

\title{CRPS-Optimal Binning for Univariate Conformal Regression}

\author{\Name{Paolo Toccaceli} \Email{paolo.toccaceli@rhul.ac.uk} \\
       \addr {Centre for Reliable Machine Learning, \\
       Royal Holloway, University of London, \\
       Egham, Surrey\\
       TW20 0EX, UK}}
										
\editor{Johan Hallberg Szabadváry, Ulf Johansson, Henrik Boström}

\begin{document}

\maketitle

\begin{center}
\small\textit{Submitted to the Conformal and Probabilistic Prediction with Applications (COPA 2026).}
\end{center}

\begin{abstract}
We propose a method for non-parametric conditional distribution estimation based on
partitioning covariate-sorted observations into contiguous bins and using the within-bin
empirical CDF as the predictive distribution.
Bin boundaries are chosen to minimise the total leave-one-out Continuous Ranked Probability
Score (LOO-CRPS), which admits a closed-form cost function with $O(n^2 \log n)$
precomputation and $O(n^2)$ storage; the globally optimal $K$-partition is recovered by a
dynamic programme in $O(n^2 K)$ time.
We select $K$ by $K$-fold cross-validation of test CRPS, which yields a U-shaped
criterion with a well-defined minimum.
Having selected $K^*$ and fitted the full-data partition, we form
a conformal prediction set based on CRPS
as the nonconformity score, which carries a finite-sample marginal coverage guarantee at
any prescribed level $\varepsilon$. The conformal prediction is transductive and data-efficient, 
as all observations are used for both partitioning and p-value calculation, with no need to 
reserve a hold-out set.
On real benchmarks against split-conformal competitors (Gaussian split conformal,
CQR, CQR-QRF, and conformalized isotonic distributional regression), the full-$n$
method produces narrower prediction intervals while maintaining near-nominal coverage;
a matched-sample comparison restricting to the same training half as the competitors is also reported.
\end{abstract}


\section{Introduction}

A fundamental problem in supervised learning is to estimate not just the conditional mean
$\mathbb{E}[Y \mid X=x]$ but the full conditional distribution $P(Y \mid X=x)$.
A distributional forecast quantifies uncertainty and therefore enables principled decision-making under
risk, hypothesis testing at a test point, and the construction of valid prediction sets.

The simplest non-parametric approach is to collect nearby training observations and use
their Empirical Cumulative Distribution Function (ECDF) as a stand-in for $P(Y \mid X=x)$.
Other approaches exist, e.g. Kernel density methods, but they have high sensitivity to bandwidth choice 
or significant computational requirements.

When the covariate is one-dimensional and the data are sorted by $x$, a natural alternative
is to partition the sorted sequence into contiguous bins and predict with the within-bin
ECDF.

The central question then is how to place bin boundaries optimally.
Informal choices, such as equal-width or equal-count bins, are oblivious to the
heteroscedasticity of the response and to the local density of the covariate.
We argue that bin boundaries should minimise a proper scoring rule evaluated on the
training data, so that the binning criterion and the predictive goal are aligned.
Among proper scoring rules, the Continuous Ranked Probability Score (CRPS) is particularly 
well suited to this purpose: it
targets the full distribution (not a single functional), it admits a closed-form
leave-one-out formula that depends only on pairwise absolute differences, and it is the
same score used in the conformal prediction step, ensuring coherence between bin selection
and prediction-set construction.

\paragraph{Contributions.}
\begin{enumerate}
  \item \textbf{Closed-form LOO-CRPS cost.}
    We derive the total leave-one-out CRPS of a bin of size $m$ as
    $\mathrm{cost}(S) = m W/(m-1)^2$, where $W = \sum_{\ell < r}|y_\ell - y_r|$ is the
    pairwise dispersion (Proposition~\ref{prop:cost}).
    This scalar is updated in $O(\log n)$ per extension step, giving $O(n^2 \log n)$
    precomputation and $O(n^2)$ storage for all subintervals.
  \item \textbf{Exact optimal partitioning via dynamic programming.}
    The additive cost structure satisfies the optimal substructure property, so a DP
    recovers the globally optimal $K$-partition in $O(n^2 K)$ time (Section~\ref{sec:dp}).
    Unlike greedy binary segmentation~\citep{AugerLawrence1989}, which fixes cut points
    sequentially and can miss the global optimum, the DP evaluates all continuations
    simultaneously and is exact --- and vastly more efficient than exhaustive search over
    all $\binom{n-1}{K-1}$ candidate $K$-partitions.
  \item \textbf{Conformal prediction sets.}
    We construct CRPS-based conformal prediction sets with finite-sample marginal coverage
    guarantees.  In the split-conformal case, convexity of the score guarantees that
    prediction sets are always connected intervals; in our transductive setting, single-interval
    structure is observed empirically in all experiments.
    The regression form of the Venn Predictor (a constant-width family of augmented ECDFs)
    is formalised as a complementary output in Appendix~\ref{app:venn}.
\end{enumerate}
A distinguishing feature of the method is that it is \emph{transductive}: all $n$
observations are used for both partitioning and conformal calibration, with no
data held out.
Split-conformal competitors must reserve a calibration set, effectively halving
the sample available for model fitting.
This data-efficiency advantage is particularly pronounced in small-sample settings.


\section{Related Work}
\label{sec:related}

\paragraph{Conformal prediction for regression.}
Conformal prediction~\citep{VovkEtAl2022} wraps any nonconformity score in a
distribution-free prediction set with finite-sample marginal coverage;
~\citet{LeiEtAl2018} provide a comprehensive treatment for the regression
setting, covering a range of nonconformity scores and establishing finite-sample guarantees.
The inductive (split-conformal) variant~\citep{Papadopoulos2002} divides training data into
a fitting half and a calibration half; the quantile of calibration nonconformity scores
determines the prediction set for a new point.
Conformalized Quantile Regression (CQR)~\citep{RomanoEtAl2019} improves efficiency in
heteroscedastic settings by using conditional quantile estimates as the base predictor and
measuring residuals relative to the estimated interval.
Mondrian conformal prediction~\citep{VovkEtAl2022} stratifies the calibration set into
categories (tax\-onomies) so that coverage holds conditionally within each stratum; our
method is a Mondrian predictor with data-adaptive bins as strata and CRPS as the
nonconformity score.
~\citet{SesiaRomano2021} propose conformal prediction based on conditional
histograms, stratifying calibration scores by the value of a one-dimensional predictor
(such as an estimated quantile); our method instead partitions on the covariate $x$ directly
and selects bin boundaries to minimise LOO-CRPS rather than conditioning on a downstream
predictor score.

\paragraph{Venn predictors.}
Venn predictors output a set of probability distributions, one for
each candidate label of the test point.
This work draws much of its inspiration from the Venn-ABERS predictors for binary
classification proposed by ~\citet{VovkPetej2014}: isotonic regression 
on an underlying scoring function implicitly identifies an optimal partition of the 
score range into reference classes, each with finite-sample calibration guarantees.
Our method mirrors this structure on the covariate axis: the DP selects CRPS-optimal
bins, defining reference classes whose within-bin ECDF best describes the local
conditional distribution.

\paragraph{Conformal predictive systems and distributional conformal prediction.}
\citet{AllenEtAl2025} establish a unifying framework: any distributional
regression procedure that is in-sample calibrated, when conformalized, yields a conformal
predictive system with out-of-sample calibration guarantees under exchangeability.
Their framework covers conformal binning, that is, partitioning calibration data into
covariate-similar groups and predicting with within-group ECDFs, and conformal
isotonic distributional regression (IDR) as canonical instances, but leaves the
bin-selection criterion unspecified.
The present paper fills this gap: we identify CRPS-optimal contiguous bins via
dynamic programming and provide a closed-form cost formula, a CV criterion for $K$,
and structural results on the resulting prediction set.
Since the within-bin ECDF is in-sample auto-calibrated by construction, our
conformal binning step is an instance of their framework and inherits the
corresponding calibration guarantee.
\citet{ChernozhukovEtAl2021} construct (approximately) conditionally
valid prediction intervals by permuting PIT residuals; their approach targets conditional
validity via a rank-based argument rather than CRPS-optimal predictive distributions.
\citet{RandahlEtAl2026} enforce coverage within pre-specified
\emph{outcome} bins to promote coverage equity across label subgroups; their partition
is on the response $y$ rather than on the covariate $x$, and the objective is coverage
equity rather than predictive sharpness.

\paragraph{Conditional distribution estimation.}
Model-free competitors include Quantile Regression Forests~\citep{Meinshausen2006}
and NGBoost~\citep{DuanEtAl2020}; unlike these, our method requires only observations
sorted by covariate and makes no distributional assumption.

\paragraph{Optimal partitioning.}
The DP recurrence (Section~\ref{sec:dp}) is structurally identical to exact segment
neighbourhood algorithms~\citep{AugerLawrence1989,KillickEtAl2012}.
The key distinction is objective: classical methods seek distributional homogeneity,
whereas we minimise predictive LOO-CRPS.

\section{Setup}
\label{sec:setup}

Let $(x_1, y_1), \ldots, (x_n, y_n) \in \mathbb{R}^2$ be a training sample.
Sort observations by covariate value so that $x_{(1)} \le x_{(2)} \le \cdots \le x_{(n)}$,
and write $y_i$ for the response paired with $x_{(i)}$.

A \emph{$K$-partition} of the sorted observations is a sequence of indices
$0 = b_0 < b_1 < \cdots < b_K = n$ defining $K$ contiguous bins
$B_k = \{b_{k-1}+1, \ldots, b_k\}$ for $k = 1, \ldots, K$.
Within bin $B_k$, the predictive distribution for a new $x \in [x_{(b_{k-1}+1)}, x_{(b_k)}]$
is taken to be the empirical CDF of $\{y_i : i \in B_k\}$.

\section{LOO-CRPS Cost of a Bin}
\label{sec:cost}

\subsection*{The CRPS: definition and geometric interpretation}

For a predictive CDF $F$ and a scalar outcome $y \in \mathbb{R}$, the
\emph{Continuous Ranked Probability Score}~\citep{GneitingRaftery2007} is the integrated squared difference
between $F$ and the CDF of a point mass at $y$:
\begin{equation}
  \label{eq:crps-integral}
  \mathrm{CRPS}(F, y)
  = \int_{-\infty}^{\infty} \bigl(F(t) - \mathbf{1}[t \ge y]\bigr)^2\,\mathrm{d}t.
\end{equation}
Here $\mathbf{1}[t \ge y]$ is the right-continuous CDF of a Dirac mass at $y$;
some references write $\mathbf{1}[t < y]$ (the left-continuous version), which
agrees everywhere except at the single point $t = y$ and leaves the integral unchanged.
Geometrically, the integrand is the squared vertical gap between $F$ and the step
at~$y$ (Figure~\ref{fig:crps_def}).
CRPS is zero if and only if $F$ is a point mass at~$y$; diffuse or mis-centred forecasts
accumulate a larger integrated gap and hence a larger score.
An equivalent energy-score form~\cite[Eq.~(21)]{GneitingRaftery2007}, convenient for computation, is
\begin{equation}
  \label{eq:crps-energy}
  \mathrm{CRPS}(F, y)
  = \mathbb{E}_F|X - y| - \tfrac{1}{2}\,\mathbb{E}_F|X - X'|,
\end{equation}
where $X, X'$ are independent draws from $F$.
The first term penalises mis-location; the second subtracts a self-dispersion penalty that
rewards sharpness, preventing the score from being minimised by a diffuse prior.
CRPS is a \emph{strictly proper} scoring rule for the class of all distributions with
finite first moment~\citep{GneitingRaftery2007}: its expected value under $P$ is uniquely
minimised by the forecast $\hat{F} = P$.

\begin{figure}[ht]
\centering
\begin{tikzpicture}[x=0.98cm, y=2.6cm]
  \draw[->] (-0.5,0) -- (5.7,0) node[right] {$t$};
  \draw[->] (0,-0.06) -- (0,1.23) node[above, font=\small] {CDF};
  \foreach \v/\lbl in {0.25/{\scriptstyle\frac{1}{4}},
                       0.50/{\scriptstyle\frac{1}{2}},
                       0.75/{\scriptstyle\frac{3}{4}},
                       1.00/{\scriptstyle 1}} {
    \draw (-0.07,\v)--(0.07,\v);
    \node[left, inner sep=1pt] at (-0.09,\v) {$\lbl$};
  }
  \fill[blue!15] (1.0,0.00) rectangle (2.0,0.25);
  \fill[blue!15] (2.0,0.00) rectangle (3.0,0.50);
  \fill[orange!22] (3.0,0.50) rectangle (3.5,1.00);
  \fill[orange!22] (3.5,0.75) rectangle (4.5,1.00);
  \draw[very thick, gray!55!black] (-0.4,0.00)--(3.0,0.00);
  \draw[very thick, gray!55!black] (3.0,1.00)--(5.5,1.00);
  \draw[dashed, gray!55!black] (3.0,0.00)--(3.0,1.00);
  \node[gray!55!black, right, font=\small] at (4.7,1.13)
    {$\mathbf{1}[t \ge y]$};
  \draw[very thick, blue!65!black] (-0.4,0.00)--(1.0,0.00);
  \draw[very thick, blue!65!black] (1.0,0.25)--(2.0,0.25);
  \draw[very thick, blue!65!black] (2.0,0.50)--(3.5,0.50);
  \draw[very thick, blue!65!black] (3.5,0.75)--(4.5,0.75);
  \draw[very thick, blue!65!black] (4.5,1.00)--(5.5,1.00);
  \foreach \x/\lo/\hi in {1.0/0.00/0.25, 2.0/0.25/0.50,
                           3.5/0.50/0.75, 4.5/0.75/1.00} {
    \draw[dashed, blue!50!black] (\x,\lo)--(\x,\hi);
    \draw[thin] (\x,-0.025)--(\x,0.025);
  }
  \node[below, font=\scriptsize] at (1.0,-0.045) {$y_1$};
  \node[below, font=\scriptsize] at (2.0,-0.045) {$y_2$};
  \node[below, font=\scriptsize] at (3.5,-0.045) {$y_3$};
  \node[below, font=\scriptsize] at (4.5,-0.045) {$y_4$};
  \node[blue!65!black, right, font=\small] at (4.7,0.86) {$\hat{F}_m(t)$};
  \node[below, font=\small] at (3.0,-0.055) {$y$};
  \draw[<->, thin, gray!60] (2.5,0.02)--(2.5,0.48);
  \node[font=\scriptsize, right] at (2.55,0.25)
    {$\hat{F}_m(t)-\mathbf{1}[t{\ge}y]$};
\end{tikzpicture}
\caption{Geometric interpretation of $\mathrm{CRPS}(\hat{F}_m, y)$ as the integral
of the squared vertical gap between the predictive CDF $\hat{F}_m$ (blue step function,
$m = 4$ atoms) and the step $\mathbf{1}[t \ge y]$ (grey) at the observed outcome~$y$.
Blue shading marks intervals where $\hat{F}_m(t) > \mathbf{1}[t \ge y]$ (too little
forecast mass above~$t$); orange marks intervals where $\hat{F}_m(t) < \mathbf{1}[t \ge y]$
(too little mass below~$t$).
CRPS $= \int (\hat{F}_m(t) - \mathbf{1}[t\ge y])^2\,\mathrm{d}t$ is large when
the forecast is mis-centred or over-dispersed relative to~$y$.}
\label{fig:crps_def}
\end{figure}

\subsection*{Empirical CDF and the LOO cost}

For a predictive CDF $\hat{F}$ consisting of $m$ equally weighted atoms and outcome $y$,
applying~\eqref{eq:crps-energy} gives
\[
  \mathrm{CRPS}(\hat{F}, y)
  = \frac{1}{m}\sum_{i=1}^m |y_i - y|
    - \frac{1}{2m^2}\sum_{i=1}^m\sum_{j=1}^m |y_i - y_j|.
\]

Let $S$ be a bin of size $m$ with response values $y_1, \ldots, y_m$.
Write
\[
  d_k = \sum_{\ell \ne k} |y_\ell - y_k|,
  \qquad
  D = \sum_{\ell \ne r} |y_\ell - y_r| = 2\sum_{\ell < r}|y_\ell - y_r|.
\]
The leave-one-out predictive distribution for observation $k$ is the ECDF of
$\{y_\ell : \ell \in S,\, \ell \ne k\}$.
Applying the CRPS formula with $m - 1$ atoms:
\begin{align*}
  \mathrm{CRPS}(\hat{F}_{S\setminus\{k\}},\, y_k)
  &= \frac{d_k}{m-1}
     - \frac{D - 2d_k}{2(m-1)^2}.
\end{align*}

\begin{proposition}
\label{prop:cost}
The total leave-one-out CRPS of bin $S$ is
\[
  \mathrm{cost}(S)
  \;=\; \sum_{k \in S}\mathrm{CRPS}(\hat{F}_{S\setminus\{k\}},\, y_k)
  \;=\; \frac{m}{2(m-1)^2}\, D
  \;=\; \frac{m}{(m-1)^2} \sum_{\ell < r,\; \ell,r\in S} |y_\ell - y_r|.
\]
\end{proposition}

\begin{proof}
See Appendix~\ref{app:proofs}.
\end{proof}

For a bin spanning sorted indices $i$ through $j$ (with $m = j - i + 1$), write
\[
  W(i,j) = \sum_{\substack{\ell < r \\ \ell,r \in \{i,\ldots,j\}}} |y_\ell - y_r|,
  \qquad
  c(i,j) = \frac{j-i+1}{(j-i)^2}\, W(i,j).
\]

\section{Dynamic Programme}
\label{sec:dp}
It turns out that it is possible to solve the optimal partitioning problem using dynamic programming. 
The DP is efficient compared to the combinatorial explosion of exhaustive search, which would enumerate
all $\binom{n-1}{K-1}$ candidate $K$-partitions, and is exact, unlike, say, greedy binary segmentation, 
which fixes cut points sequentially and can miss the global optimum.
The problem of minimizing the total LOO-CRPS of a $K$-partition can be solved by a DP because the cost function
can be expressed in a recurrence that exhibits the optimal substructure property, as we will show below.

Define $\mathrm{dp}[k][j]$ as the minimum total LOO-CRPS achievable by partitioning
observations $1, \ldots, j$ into exactly $k$ contiguous bins.

\paragraph{Minimum bin size.}
The LOO cost $c(i,j)$ requires at least two observations: for a singleton bin
the leave-one-out distribution is empty, so the LOO-CRPS is undefined.
We set $c(i,i) = +\infty$ by convention, which causes the DP to exclude
singleton bins automatically.
The index ranges below implicitly assume $j - i \ge 1$ (i.e.\ $m \ge 2$).

\paragraph{Base case.}
\[
  \mathrm{dp}[1][j] = c(1, j), \qquad j = 2, \ldots, n.
\]

\paragraph{Recurrence.}
If $\mathrm{dp}[k-1][i]$ gives the cost of the optimal partition of the first
$i$ observations into $k-1$ bins, then to partition $1,\ldots,j$ into $k$ bins
optimally it suffices to try every candidate last-bin start $i+1$:
the last bin covers observations $i+1,\ldots,j$ with cost $c(i+1,j)$, and the
preceding $i$ observations are already optimally split at cost $\mathrm{dp}[k-1][i]$.
Taking the minimum over $i$ gives the recurrence ($k \ge 2$, $j \ge k$):
\[
  \mathrm{dp}[k][j]
  = \min_{k-1 \,\le\, i \,<\, j}
    \Bigl\{\mathrm{dp}[k-1][i] + c(i+1,\, j)\Bigr\}.
\]

\paragraph{Solution.} The optimal $K$-partition cost is $\mathrm{dp}[K][n]$.
Bin boundaries are recovered by backtracking the argmin at each step.

The complete pseudocode is given as Algorithm~\ref{alg:dp} in Appendix~\ref{app:impl};
a Python implementation is also available in the \texttt{crpsconfreg} package.

\section{Precomputation and Complexity}
\label{sec:precomp}

\paragraph{Precomputing $W(i,j)$.}
Fix left endpoint $i$ and scan $j = i, i+1, \ldots, n$.
Let $m = j - i + 1$ denote the current number of elements in the bin $[i, j]$.
Adding $y_{j+1}$ to the set $\{y_i, \ldots, y_j\}$ increases the pairwise
dispersion $W$ by
$\Delta W = \sum_{\ell=i}^{j}|y_\ell - y_{j+1}|$.
To evaluate this sum in $O(\log n)$ without iterating over all $m$ elements,
split by sign.
Let $r$ denote the number of existing values $\le y_{j+1}$
(the \emph{rank} of $y_{j+1}$ in the current set),
$S_{\le} = \sum_{\ell:\, y_\ell \le y_{j+1}} y_\ell$,
and $S_{>} = \sum_{\ell:\, y_\ell > y_{j+1}} y_\ell$.
Then
\begin{align*}
  \Delta W
  &= \sum_{\ell:\, y_\ell \le y_{j+1}} (y_{j+1} - y_\ell)
   \;+\; \sum_{\ell:\, y_\ell > y_{j+1}} (y_\ell - y_{j+1}) \\
  &= y_{j+1} \cdot r - S_{\le}
   \;+\; S_{>} - y_{j+1}\cdot(m - r).
\end{align*}
The quantities $r$ and $S_{\le}$ are \emph{prefix queries}: a count and a sum
over all stored values up to a query point. $S_{>}$ follows as the running total minus
$S_{\le}$.
Prefix queries can be supported efficiently using Fenwick trees 
(Binary Indexed Trees)~\citep{Fenwick1994}, for which searches and insertions 
can be performed in $O(\log n)$ time. We use two Fenwick trees (one for counts, one 
for values) indexed by value rank.  
Since there are $O(n)$ left endpoints and $O(n)$ extensions for each,
the total precomputation is $O(n^2 \log n)$ with $O(n^2)$ storage.
Data-structure trade-offs are discussed in
Appendix~\ref{app:impl}.

\paragraph{DP.}
The recurrence evaluates in $O(n^2 K)$ time.
There are $K$ layers and $O(n)$ entries per layer; filling each entry
$\mathrm{dp}[k][j]$ requires scanning $O(n)$ candidate split points $i$,
with each lookup of $c(i+1,j)$ taking $O(1)$ from the precomputed table.

\paragraph{Complexity tightness.}
In some cases, the efficiency of the DP can be further improved by exploiting the structure of the cost function.
A standard result of algorithm theory often referred to as the Knuth--Yao theorem~\citep{Knuth1971,Yao1980} 
indicates that if the cost function satisfies a quadrangle inequality, then the DP can be further sped up. 
Unfortunately, the LOO-CRPS cost violates the quadrangle inequality (explicit counterexample in
Appendix~\ref{app:qi}), so, at least under this respect, $O(n^2 K)$ is tight: no Knuth--Yao speedup applies.

\section{Selecting \texorpdfstring{$K$}{K}}
\label{sec:K}

The DP finds the optimal partition, i.e. the one that minimises total LOO-CRPS, for a \emph{fixed} $K$.
A natural approach to selecting $K$ is to treat the total LOO-CRPS $\mathrm{dp}[K][n]$
as a function of $K$ and pick its minimum\footnote{
As an alternative, one could place a Dirichlet process prior over the partition; the number of clusters
is then determined by the concentration parameter rather than cross-validation.
While principled, this approach does not in general minimise CRPS, and the choice
of concentration parameter introduces a separate hyperparameter.
We use CV as a simpler, assumption-free alternative.}. 
As it is customary, we perform cross-validation to strike a balance between underfitting and overfitting.
The sorted observations are assigned to folds by interleaving:
observation $i$ (in the $x$-sorted order) goes to fold $i \bmod F$,
so that every fold inherits the $x$-sorted structure (note that we use $F$ for the number of folds to avoid
confusion with the number of bins $K$).
For each fold $f \in \{0,\ldots,F-1\}$, let $\mathcal{T}_f$ and
$\mathcal{V}_f$ denote the training and test subsets.
For each candidate $K$, find the optimal $K$-partition $\hat{\mathcal{P}}_K^{(f)}$
on $\mathcal{T}_f$ using the DP, then evaluate the test CRPS on $\mathcal{V}_f$:
\[
  \mathrm{TestCRPS}_f(K)
  = \frac{1}{|\mathcal{V}_f|}
    \sum_{(x_i, y_i) \in \mathcal{V}_f}
    \mathrm{CRPS}\!\left(\hat{F}_{b(x_i)}^{(f)},\, y_i\right),
\]
where $b(x_i)$ is the bin in $\hat{\mathcal{P}}_K^{(f)}$ that contains $x_i$ and
$\hat{F}_{b(x_i)}^{(f)}$ is the ECDF of the training $y$-values in that bin.
The final criterion averages across folds:
\[
  \overline{\mathrm{TestCRPS}}(K) = \frac{1}{F}\sum_{f=0}^{F-1} \mathrm{TestCRPS}_f(K),
  \qquad
  K^* = \operatorname*{arg\,min}_{K=1,\ldots,K_{\max}} \overline{\mathrm{TestCRPS}}(K).
\]
We use $F=5$ folds and set $K_{\max} = \lfloor n/10 \rfloor$ throughout,
ensuring the minimum expected bin size on each training fold is at least $8$;
values larger than $\lfloor n/10 \rfloor$ typically produce nearly empty bins
and are penalised heavily by test CRPS regardless.

\paragraph{Running example.}
To illustrate the methods of this and subsequent sections, we use a
heteroscedastic synthetic dataset with $n = 1000$ observations:
$X_i \sim \mathrm{Uniform}(0, 3)$,
$Y_i \mid X_i = x \sim \mathcal{N}(3x,\,(1+x)^2)$.
The conditional mean grows from $0$ to $9$ and the conditional standard deviation
from $1$ to $4$ over $[0,3]$, giving a $16\times$ increase in variance.
Observations are sorted by $x$ before all subsequent steps.
A full analysis is given in Appendix~\ref{app:example}.

\begin{figure}[h]
\centering
\includegraphics[width=0.78\textwidth]{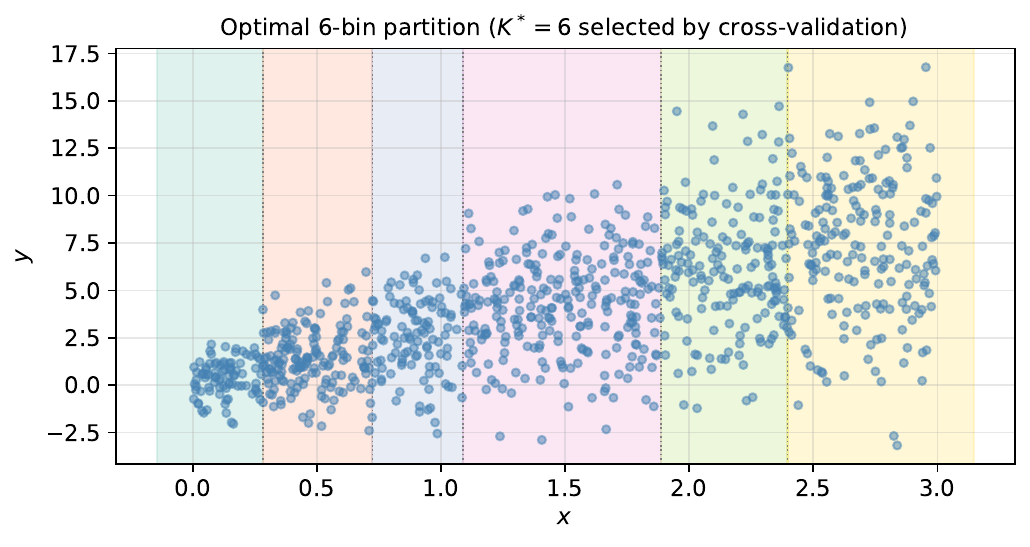}
\caption{The synthetic dataset with the optimal $6$-bin partition ($K^* = 6$).
Shaded regions correspond to the six bins; dotted vertical lines mark the bin boundaries
(midpoints between adjacent training observations).}
\label{fig:partition}
\end{figure}

%
%

\begin{figure}[h]
\centering
\includegraphics[width=\textwidth]{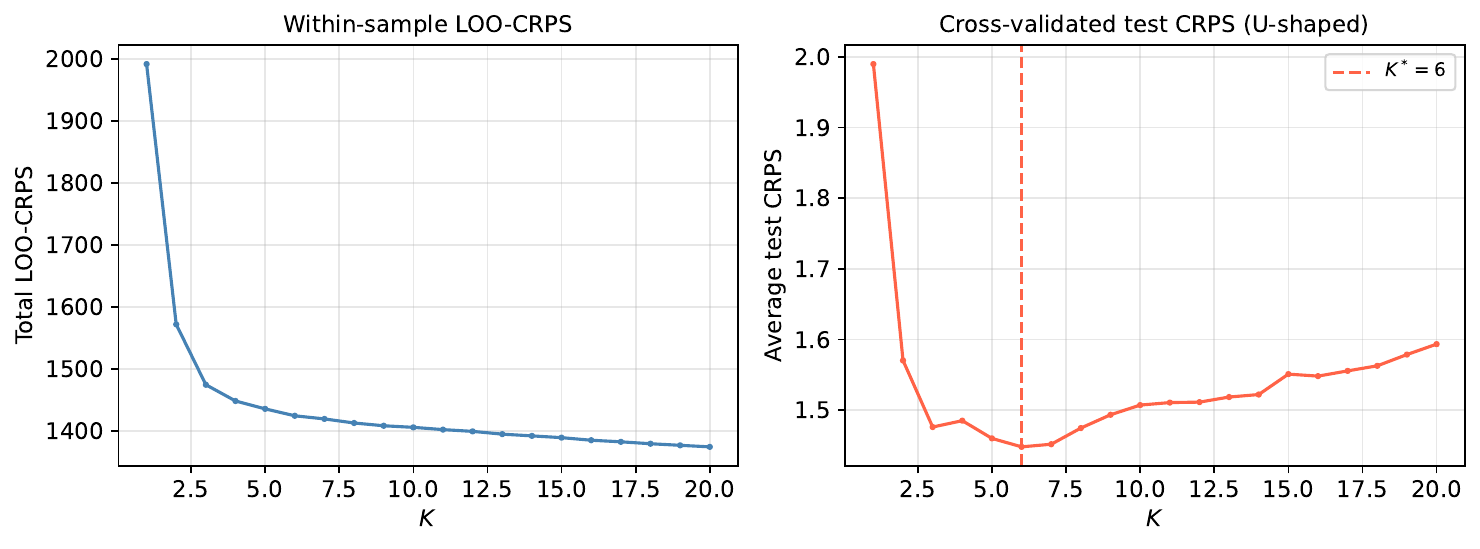}
\caption{Within-sample LOO-CRPS (left) and cross-validated test CRPS (right) as functions
of $K$ on the running example.
The within-sample criterion is nearly monotone decreasing, confirming its susceptibility to optimism bias
~\citep{HastieTibshiraniFriedman2009}.
The test CRPS has a clear U-shape with minimum at $K^*=6$.
Note that the left and right $y$-axis scales differ: the left shows a total (sum over all observations),
the right an average (per held-out test point).}
\label{fig:kselect}
\end{figure}

\section{Predictive Distributions: Conformal Inference}
\label{sec:pred}

Having selected $K^*$ and fitted the partition on all data, each test point $x^*$ falls
in some bin $B_k$ with $m$ training observations $y_1,\ldots,y_m$.
The partition enables two complementary predictive outputs: a \emph{Venn prediction
band}---a constant-width family of augmented ECDFs that is valid as a multiprobabilistic
prediction under exchangeability (derivation and figure in Appendix~\ref{app:venn})---and
a \emph{conformal prediction set}, which adapts to local density and is the focus of
this section.

\subsection{Conformal Prediction Set}
\label{sec:conformal-set}

We briefly recall the key concepts of conformal prediction~\citep{VovkEtAl2022} 
in the context of our setting and we show how to obtain valid prediction sets.

\paragraph{Nonconformity score.}
For test candidate $y_h$, define
\[
  \alpha(y_h) = \mathrm{CRPS}(\hat{F}_m,\, y_h)
  = \frac{1}{m}\sum_{i=1}^m |y_i - y_h| - \frac{W}{m^2},
\]
where $W = \sum_{i<j}|y_i - y_j|$ is the pairwise dispersion of the training bin
(independent of $y_h$).
For each training observation $y_j$ ($j = 1,\ldots,m$) in the augmented set
$\{y_1,\ldots,y_m,y_h\}$, define the leave-one-out nonconformity score
\[
  \alpha_j(y_h) = \mathrm{CRPS}(F^{(-j)}_{m+1},\, y_j),
\]
where $F^{(-j)}_{m+1}$ is the ECDF of $\{y_1,\ldots,y_m,y_h\}\setminus\{y_j\}$.
Write $\alpha_{m+1}(y_h) \equiv \alpha(y_h)$ for the test score.

\paragraph{Conformal p-value.}
\[
  p(y_h) = \frac{1}{m+1}\,\#\bigl\{j \in \{1,\ldots,m+1\} :
  \alpha_j(y_h) \ge \alpha(y_h)\bigr\}.
\]

\paragraph{Prediction set.}
\[
  \Gamma^\varepsilon = \{y_h \in \mathbb{R} : p(y_h) > \varepsilon\}.
\]

\begin{proposition}
\label{prop:coverage}
Under exchangeability of $(y_1,\ldots,y_m,y^*)$,\; $\Pr(y^* \in \Gamma^\varepsilon) \ge 1-\varepsilon$.
\end{proposition}

For a proof, refer to the standard conformal prediction literature~\citep{VovkEtAl2022}.


\paragraph{Remark on exchangeability.}
A data-dependent partition mildly violates the exchangeability condition underlying
Proposition~\ref{prop:coverage}: the DP uses all $y$-values to set boundaries, so
conditioning on a bin's membership acts as a $y$-dependent filter.
The violation is less severe as bin sizes grow, and Theorem 2.1 in ~\citep{AllenEtAl2025}
provide formal support: auto-calibrated within-bin predictors inherit conformal
guarantees even for data-dependent partitions.
Further discussion on exchangeability is in Appendix~\ref{app:exch}.



\paragraph{The suitability of LOO-CRPS as a nonconformity score}

Several properties make the LOO-CRPS a natural and principled nonconformity measure
in this setting.

\textit{Properness implies sensitivity to genuine nonconformity.}
Because CRPS is strictly proper, the expected score $\mathbb{E}_P[\mathrm{CRPS}(\hat{F},Y)]$
is uniquely minimised when $\hat{F} = P$.
Consequently, $\alpha(y_h) = \mathrm{CRPS}(\hat{F}_m, y_h)$ is large precisely when
$y_h$ is surprising under the within-bin training distribution, not merely when it is
far from some summary statistic such as the mean or median.
Nonconformity scores based on absolute residuals $|y_h - \hat{\mu}|$
(or quantile residuals) implicitly assume a parametric location or quantile model for the
bin; CRPS makes no such assumption, which matters when the within-bin distribution is
skewed or multimodal.

\textit{The LOO structure guarantees exchangeability.}
Conformal validity requires that the scores $\alpha_1(y_h), \ldots, \alpha_m(y_h),
\alpha(y_h)$ be \emph{exchangeable} under exchangeability of $(y_1,\ldots,y_m,y^*)$.
This holds here because every score is computed in exactly the same way: each $\alpha_j$
is the CRPS of the $m$-atom ECDF of the remaining $m$ elements of
$\{y_1,\ldots,y_m,y_h\}$ evaluated at $y_j$, and $\alpha(y_h)$ is the same
quantity for the test element.
If instead we used the full-sample $\hat{F}_m$ (without the LOO adjustment) to score
both training and test points, the training scores
$\mathrm{CRPS}(\hat{F}_m, y_j)$ and the test score $\mathrm{CRPS}(\hat{F}_m, y^*)$
would not be on equal footing: the training ECDF was fitted using $y_j$, but not
using $y^*$, breaking the symmetry on which conformal coverage rests.

\textit{Coherence with the binning criterion.}
The DP selects bin boundaries by minimising total LOO-CRPS on the training set.
The conformal step then evaluates the test point by the same quantity: the marginal
LOO-CRPS of adding $y^*$ to the selected bin.
This coherence means the bin is already optimised for the task of distributional
prediction, and the prediction set $\Gamma^\varepsilon$ has the natural interpretation
as the set of test values whose inclusion would not increase the bin's per-observation
LOO-CRPS beyond the level seen at the training points (see also Section~\ref{sec:conn}).
Figure~\ref{fig:pvalue} illustrates the p-value curve at three test locations.

\begin{figure}[h]
\centering
\includegraphics[width=\textwidth]{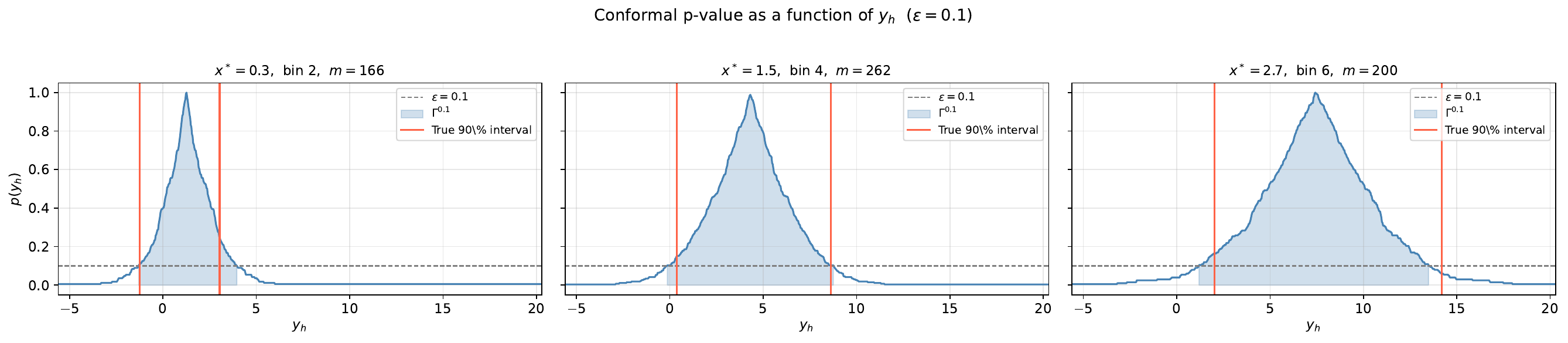}
\caption{Conformal p-value $p(y_h)$ as a function of the candidate response $y_h$ at
three test points $x^* \in \{0.3, 1.5, 2.7\}$ for $\varepsilon = 0.10$.
The shaded region is the prediction set $\Gamma^{0.10}$; vertical red lines mark the
true $90\%$ interval under $\mathcal{N}(3x^*,(1+x^*)^2)$.
The p-value curve is unimodal (each monotone piece is convex, since the underlying
nonconformity score $\alpha(y_h)$ is convex), yielding a single connected prediction set
in each case.}
\label{fig:pvalue}
\end{figure}

\paragraph{Computational advantages.}
CRPS appears unique among strictly proper scoring rules in admitting a closed-form
LOO sum computable in $O(\log n)$ per bin extension; log score and Brier score
are not known to admit analogous closed forms, and their straightforward evaluation
would require $O(m\log m)$ or more per bin extension, compromising the $O(n^2 \log n)$ precomputation.
CRPS also serves double duty as the conformal nonconformity score, aligning bin
selection with prediction-set construction.
A reformulation of the cross-validated criterion in terms of the Cram\'{e}r
distance---which generalises CRPS to distributional observations---is given
in Appendix~\ref{app:scoring}.

\subsection{Connection to the DP Cost}
\label{sec:conn}

The sum of all $m+1$ nonconformity scores in the augmented set equals the total
LOO-CRPS of that set:
\[
  \sum_{j=1}^{m+1} \alpha_j(y_h)
  = \mathrm{cost}\bigl(\{y_1,\ldots,y_m,y_h\}\bigr)
  = \frac{m+1}{m^2}\,W\!\bigl(\{y_1,\ldots,y_m,y_h\}\bigr).
\]
The test score $\alpha(y_h) = \mathrm{CRPS}(\hat{F}_m, y_h)$ measures how much $y_h$
increases the mean absolute deviation from the bin; it is large when $y_h$ is far from
the bulk of $y_1,\ldots,y_m$.
The prediction set $\Gamma^\varepsilon$ therefore consists of values whose addition
would not dramatically inflate the DP cost of bin $B_k$.

Figure~\ref{fig:fan} shows the resulting prediction intervals across the covariate range.

\begin{figure}[h]
\centering
\includegraphics[width=0.85\textwidth]{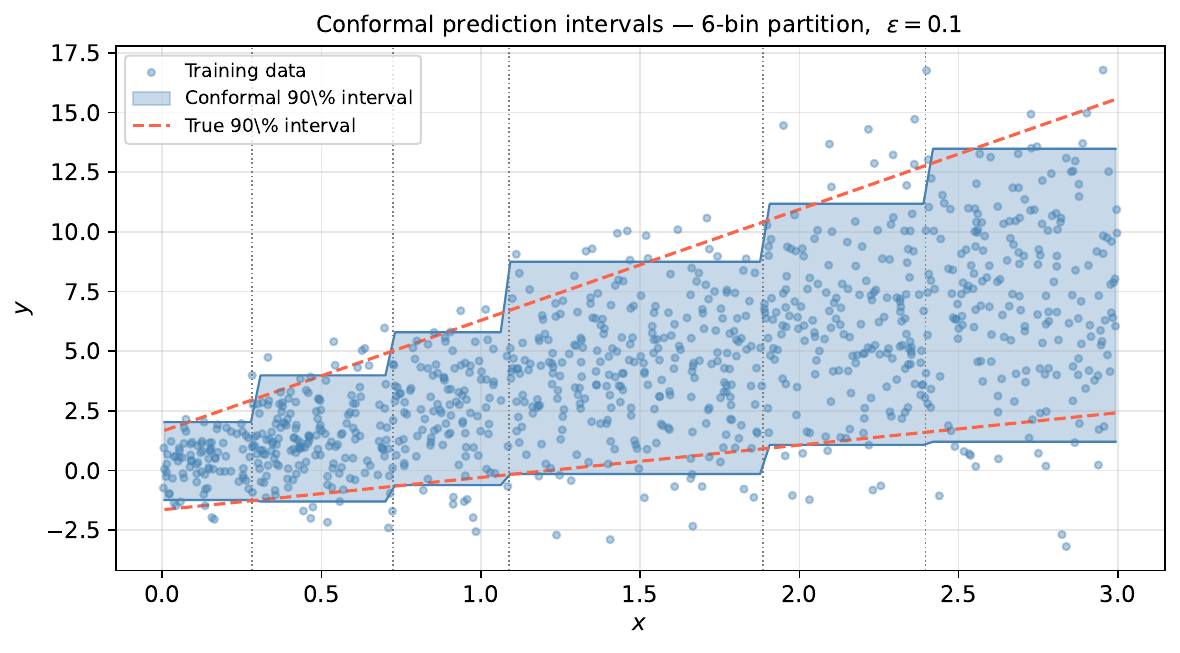}
\caption{Conformal $90\%$ prediction intervals (shaded blue) and true $90\%$ intervals
(dashed red) across the covariate range, with bin boundaries marked by dotted vertical lines.
The interval width adapts to the within-bin spread, tracking the increasing conditional
variance of the heteroscedastic data-generating process.}
\label{fig:fan}
\end{figure}

\subsection{Interval Structure and Non-Convex Extensions}
\label{sec:nonconvex}

\paragraph{$\Gamma^\varepsilon$ is approximately a single interval.}
The convexity of the CRPS nonconformity score is not incidental; it is a structural
consequence of measuring average distance to the training distribution.
Writing
\[
  \alpha(y_h)
  = \frac{1}{m}\sum_{i=1}^m |y_i - y_h| - \frac{W}{m^2},
\]
the second term is constant in $y_h$, and the first is a sum of convex functions
$|y_i - y_h|$, hence convex with a unique minimum at the empirical median of
$\{y_1,\ldots,y_m\}$.
More generally, any score of the form $\mathbb{E}_{\hat{F}_m}[\ell(y_h,Y)]$ with
$\ell$ convex in its first argument (e.g.\ squared error, absolute error, CRPS) inherits
this property.
Consequently, the sublevel set $\{y_h : \alpha(y_h) \le c\}$ is a closed interval
for every $c \ge 0$.

In a split-conformal predictor the training scores do not depend on $y_h$, so
$p(y_h)$ is non-increasing in $\alpha(y_h)$ and convexity of $\alpha$
directly implies that $\Gamma^\varepsilon$ is a connected interval.
In the present transductive (full-data) setting all $m+1$ scores depend on
$y_h$, and their relative ranking can in principle change as $y_h$ varies,
so the split-conformal argument does not apply directly.
We conjecture (and, to be clear, this remains an open problem) that 
$\Gamma^\varepsilon$ remains a connected interval whenever
$\alpha(y_h)$ is convex; in all our experiments (synthetic, Old Faithful,
motorcycle; $m$ ranging from $22$ to $262$) the prediction set is a single
interval without exception.

The consequence for multimodal bins is concrete.
Suppose the training responses in $B_k$ are bimodal with well-separated modes at
$\mu_1 \ll \mu_2$.
The empirical median lies in the inter-modal gap; $\alpha(y_h)$ attains its minimum
there.
The resulting $\Gamma^\varepsilon$ spans both modes and the low-density gap between
them: it is valid (marginal coverage is guaranteed) and correctly wide, but
informationally wasteful as it assigns coverage to a region of near-zero probability
mass.

\paragraph{A bandwidth-free non-convex alternative: the $k$-NN score.}
Non-contiguous prediction sets require a nonconformity score that is non-convex in
$y_h$, i.e.\ one that is simultaneously small near each mode and large in the
inter-modal gap.
Density-based scores such as $-\log \hat{f}(y_h)$ achieve this but require a bandwidth.
In one dimension, the $k$-nearest-neighbour ($k$-NN) distance provides a
bandwidth-free alternative.
Define
\[
  \alpha^{(k)}(y_h)
  = d_{(k)}\!\bigl(y_h,\;\{y_1,\ldots,y_m\}\bigr),
\]
the $k$-th smallest value of $|y_h - y_i|$ over $i = 1,\ldots,m$.
For $k=1$ this is simply $\min_i |y_h - y_i|$, which has a local minimum at every
training point and is large precisely where the training distribution is sparse.

The LOO version for the conformal construction is defined symmetrically in the augmented
set $\{y_1,\ldots,y_m,y_h\}$: for each training observation $y_j$,
\[
  \alpha^{(1)}_j(y_h)
  = \min\!\Bigl(\min_{i \ne j}|y_j - y_i|,\;|y_j - y_h|\Bigr),
\]
and $\alpha^{(1)}_{m+1}(y_h) \equiv \alpha^{(1)}(y_h)$.
Under exchangeability of $(y_1,\ldots,y_m,y^*)$, if we define as the conformal p-value
\[
  p^{(1)}(y_h)
  = \frac{1}{m+1}\,\#\bigl\{j : \alpha^{(1)}_j(y_h) \ge \alpha^{(1)}(y_h)\bigr\}
\]
then Proposition~2 holds without modification.

The prediction set $\Gamma^{\varepsilon,(1)} = \{y_h : p^{(1)}(y_h) > \varepsilon\}$ is
approximately the union
\[
  \bigcup_{i=1}^m\bigl[y_i - c_\varepsilon,\;y_i + c_\varepsilon\bigr],
\]
where $c_\varepsilon$ is the $(1-\varepsilon)$-quantile of the training LOO scores
(exact in the limit $m \to \infty$ where the $y_h$-dependence of $\alpha_j(y_h)$
vanishes).
For a bimodal distribution, training points cluster near both modes, so $c_\varepsilon$
is set by the intra-cluster spacing; if this is smaller than half the inter-modal gap,
the prediction set consists of two disjoint intervals.

The effective resolution is determined purely by the data; no bandwidth is specified
by the user.\footnote{It was pointed out to the author that indeed this adaptivity is
rooted in a classical result from
one-dimensional order statistics: the 1-NN distance is a consistent density estimator
without smoothing, since $m \cdot \min_i|y_h - Y_i|$ converges in distribution to
$\mathrm{Exp}(f(y_h))$ when $Y_1,\ldots,Y_m \sim f$~\citep{DevroyeGyorfi1985}.}
The conformal threshold $c_\varepsilon$ plays the role of a density threshold,
automatically calibrated for $1-\varepsilon$ coverage.
The resulting $\Gamma^{\varepsilon,(1)}$ approximates a non-parametric highest-density region
(HDR) of the within-bin empirical distribution.

\paragraph{Binning remains LOO-CRPS optimal.}
The $k$-NN modification applies only to the conformal step; the DP binning continues
to use the LOO-CRPS cost of Section~\ref{sec:cost}.
We argue that this separation is principled: the binning objective is to group covariate-sorted
observations so that the within-bin ECDF is a good predictive distribution for the
local conditional $P(Y \mid X = x)$, which is a question of predictive accuracy for
which LOO-CRPS is a natural criterion regardless of whether the within-bin
distribution is unimodal or multimodal.

Using $\sum_{j \in B_k} \alpha^{(1)}_j$ as the DP cost would be computationally
feasible (the cost is additive over bins, so optimal substructure is preserved), but
it would measure local $y$-density rather than predictive quality, and it would
exhibit the same optimism bias as within-sample LOO-CRPS: a size-2 bin with
$|y_1 - y_2| \approx 0$ has near-zero 1-NN cost, driving the DP to
over-partition.
The two steps therefore use different criteria for different purposes: LOO-CRPS for
reference-class selection, and either CRPS or $k$-NN for conformal evaluation.

\paragraph{Synthetic illustration.}
Figure~\ref{fig:bimodal_fan} illustrates the contrast on a full-regression bimodal
mixture DGP with $n=600$ training observations:
$Y \mid X{=}x \;\sim\; 0.5\,\mathcal{N}(x{-}1.5,\,0.3^2) + 0.5\,\mathcal{N}(x{+}1.5,\,0.3^2)$,
$X \sim \mathrm{Uniform}(0,6)$.
The two conditional modes run as parallel tracks separated by $3$ units across the
covariate range.
Cross-validation selects $K^*=6$ bins; both panels use the same partition.
The \emph{left} panel shows CRPS conformal intervals: for each bin the prediction set
is a single wide rectangle spanning both mode tracks and the low-density gap
between them.
The \emph{right} panel shows $k$-NN conformal intervals with $k^*=7$, selected by
cross-validating average prediction-set length: for each bin the set narrows to
intervals tracking each mode track, excluding the inter-modal gap.
Both scores carry the coverage guarantee; the visible narrowing in the
right panel corresponds to the efficiency gain from matching the prediction set to the
within-bin conditional support.

\begin{figure}[h]
\centering
\includegraphics[width=\textwidth]{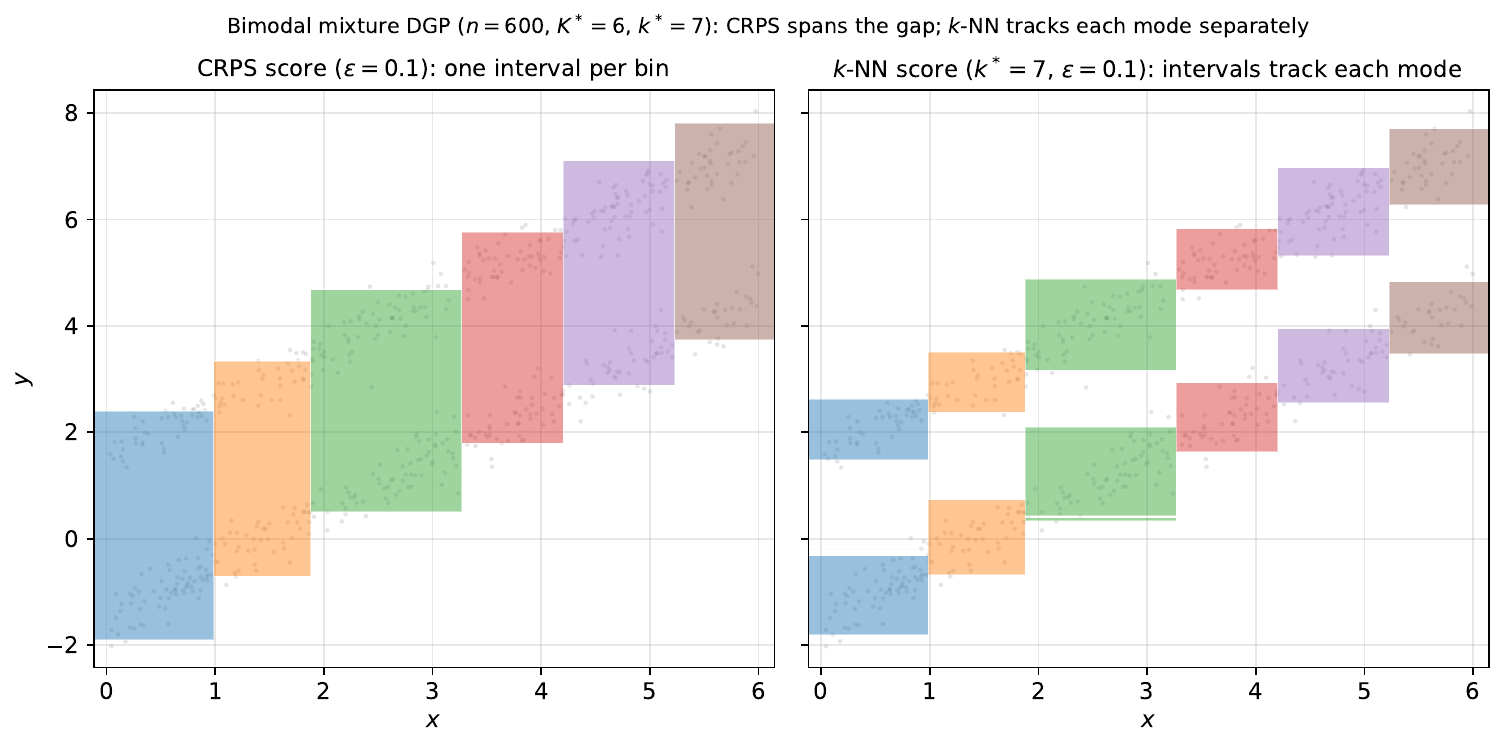}
\caption{Bimodal mixture DGP ($n=600$, $\varepsilon=0.10$, $K^*=6$).
  Shaded rectangles show the conformal prediction set for each bin;
  grey dots are training observations.
  \emph{Left:} CRPS score — one wide interval per bin, spanning both mode tracks.
  \emph{Right:} $k$-NN score ($k^*=7$, CV-selected) — narrow intervals tracking each
  mode; the low-density inter-modal gap is excluded.}
\label{fig:bimodal_fan}
\end{figure}

\subsection{Approximate Exchangeability and the Bias--Variance Trade-off in \texorpdfstring{$K$}{K}}
\label{sec:bias-variance}

Wide bins (small $K$) average over a heterogeneous region, hurting exchangeability;
narrow bins (large $K$) improve exchangeability but yield coarser intervals, since
the conformal p-value grid is $\{1/(m+1),\ldots,1\}$.
The CV criterion of Section~\ref{sec:K} penalises both failure modes and yields
a genuine U-shaped optimum (see Appendix~\ref{app:exch} for the full analysis).

\section{Real-Data Experiments}
\label{sec:realdata}

We evaluate the method on two real datasets that stress different aspects of
distributional non-stationarity: a bimodal conditional distribution and a
strongly heteroscedastic one.
All experiments use $\varepsilon = 0.10$ (nominal 90\% coverage).
Competitors are evaluated over $R=200$ random 50/50 train/calibration splits,
with standard errors reported throughout.

\paragraph{Competitors.}
\begin{itemize}
  \item \textbf{Gaussian split conformal.} OLS fit on the training half;
        absolute residuals on the calibration half determine the
        constant-width prediction interval.
  \item \textbf{CQR (cubic).} Conformalized Quantile Regression~\citep{RomanoEtAl2019}
        with cubic polynomial base quantile regressors ($\tau = 0.05$ and $0.95$),
        calibrated on the same held-out half.
  \item \textbf{CQR-QRF.} Quantile Regression Forest~\citep{Meinshausen2006} with 500
        trees fitted on the training half, conformalized with the
        $\max(q_{\alpha/2}(x)-y,\; y-q_{1-\alpha/2}(x))$ score (same as CQR)
        calibrated on the held-out half.
  \item \textbf{CQR-IDR.} Isotonic Distributional Regression~\citep{HenziEtAl2021}
        fitted on the training half, with quantile predictions at levels
        $\alpha/2$ and $1-\alpha/2$ conformalized via the same CQR score
        on the calibration half.
\end{itemize}
Our method is transductive: all $n$ observations are used for both partitioning
and conformal calibration (full-data conformal with the within-bin ECDF as
the nonconformity score), with no data held out.
Split-conformal competitors must reserve a calibration set, effectively halving
the sample available for model fitting.
This data-efficiency advantage is inherent to full-data conformal
methods~\citep{Barber2021}.
To disentangle the effect of the method itself from this sample-size advantage,
each table below reports two rows for our method: the \emph{full-$n$} row
reflects the method as deployed (transductive, no data splitting), while the
\emph{$n/2$} row handicaps it to the same training set available to competitors,
providing a controlled comparison.

\subsection{Old Faithful: bimodal conditional distribution}

The \texttt{faithful} dataset ($n = 272$) records eruption duration and
waiting time between eruptions of the Old Faithful geyser.
We pick $x = \text{waiting time (min)}$ as the predictor and
$y = \text{eruption duration (min)}$ as the response.
The marginal distribution of eruption duration is bimodal:
short eruptions ($\approx 2$ min) and long eruptions ($\approx 4.5$ min),
with the mixture proportion shifting as a function of waiting time.
This is a natural stress-test for methods that assume a unimodal conditional distribution.

Cross-validation selects $K^* = 4$, with bin boundaries at
waiting times $63.0$, $67.5$, and $71.5$ minutes.
Figure~\ref{fig:faithful_partition} (left) shows the partition;
Figure~\ref{fig:faithful_partition} (right) shows the within-bin ECDFs.
The two outer bins have clearly unimodal within-bin distributions (short
eruptions for short waits, long eruptions for long waits), while the two
inner bins capture the transition region with finer resolution.

Figure~\ref{fig:faithful_intervals} compares 90\% prediction intervals across
methods.
Gaussian split conformal and CQR produce intervals of roughly constant width
across all waiting times, whereas our method adapts: narrower in both
regimes where the conditional distribution is concentrated, and wide only in
the transition region where the within-bin ECDF spans both modes.

Table~\ref{tab:faithful_coverage} reports coverage and mean interval width
averaged over $R=200$ random 50/50 splits.
In the matched-sample comparison (top block), our method ($n/2$, italic row)
achieves mean width $1.27$ min with slightly sub-nominal coverage
($88.5 \pm 0.3\%$), narrower than all competitors: Gaussian split conformal
($1.68$ min), CQR ($1.49$ min), CQR-IDR ($1.29$ min), and CQR-QRF ($1.33$ min).
When all $n$ observations are used (bold row, below the line), our method achieves
$90.6 \pm 0.1\%$ coverage with mean width $1.22$ min.
All split-conformal methods exceed nominal coverage ($91.2$--$91.4\%$).

\begin{figure}[ht]
\centering
\includegraphics[width=\textwidth]{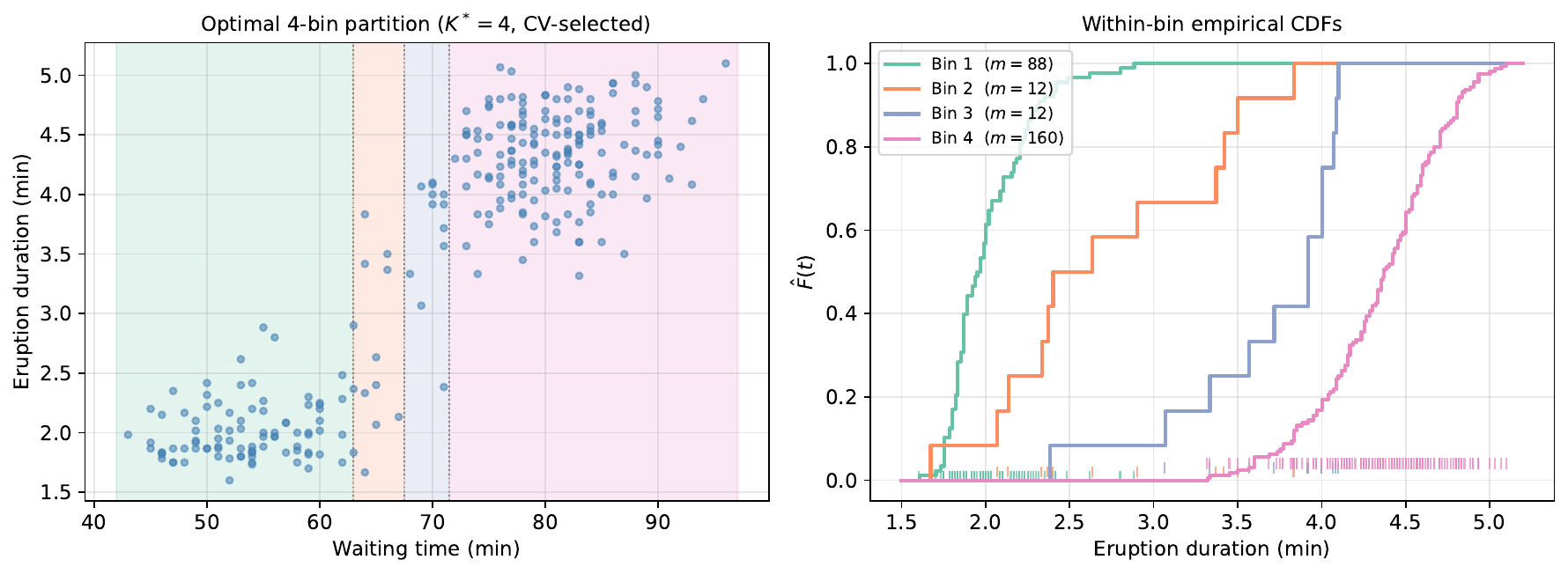}
\caption{Old Faithful. \emph{Left:} scatter of eruption duration vs.\ waiting
  time with the optimal 4-bin partition.
  Boundaries at $63.0$, $67.5$, and $71.5$ minutes resolve the transition
  between the short-eruption and long-eruption regimes.
  \emph{Right:} within-bin empirical CDFs; the outer bins are unimodal,
  confirming the partition captures the regime structure.}
\label{fig:faithful_partition}
\end{figure}

\begin{figure}[ht]
\centering
\includegraphics[width=\textwidth]{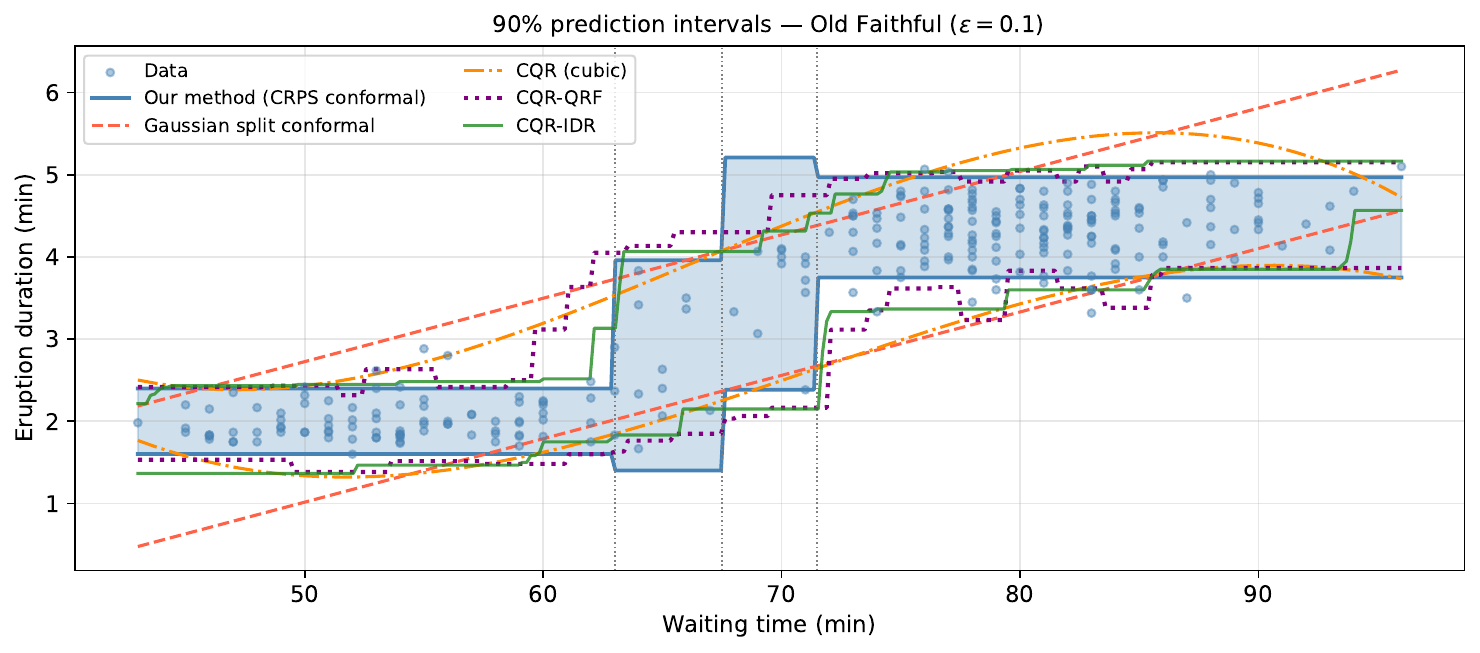}
\caption{Old Faithful: 90\% prediction intervals from five methods.
  Our CRPS-based conformal intervals adapt to the local conditional spread
  within each bin, producing narrower intervals than all competitors in both regimes.
  In the transition region around $67$--$70$ minutes, the proposed method's interval
  is deliberately wide: the single bin spanning both eruption modes must cover the
  full bimodal spread, an honest consequence of the partition geometry.}
\label{fig:faithful_intervals}
\end{figure}

\begin{table}[ht]
\centering
\begin{tabular}{lcc}
\hline
Method & Coverage (\%) & Mean width (min) \\
\hline
\textit{Our method ($n/2$)} & $88.5 \pm 0.3$ & $1.270 \pm 0.010$ \\
Gaussian split conformal & $91.2 \pm 0.0$ & $1.683 \pm 0.006$ \\
CQR (cubic) & $91.2 \pm 0.0$ & $1.490 \pm 0.006$ \\
CQR-QRF & $91.4 \pm 0.0$ & $1.333 \pm 0.006$ \\
CQR-IDR & $91.4 \pm 0.0$ & $1.294 \pm 0.007$ \\
\hline
\textbf{Our method (full $n$)} & $90.6 \pm 0.1$ & $1.217 \pm 0.002$ \\
\hline
\end{tabular}
\caption{Old Faithful ($n=272$): empirical coverage and mean width of
  prediction intervals at nominal level $1-\varepsilon = 0.90$,
  averaged over $R=200$ random 50/50 splits ($\pm$ one standard error).
  \emph{Top block}: all methods use the same training half ($n/2$);
  \textit{Our method ($n/2$)} is directly comparable to the competitors.
  \emph{Below the line}: \textbf{Our method (full $n$)} uses all $n$
  observations, a design advantage inherent to full-data conformal methods.}
\label{tab:faithful_coverage}
\end{table}

\subsection{Motorcycle accident: heteroscedastic benchmark}

The \texttt{mcycle} dataset ($n = 133$) records head acceleration of a
motorcycle dummy as a function of time after a simulated impact.
The response variance changes dramatically: near-zero before
$\approx 15$ ms, explosive in the $15$--$30$ ms deformation phase,
and moderating thereafter.
This is the standard benchmark for heteroscedastic prediction intervals
in the nonparametric regression literature.

Cross-validation selects $K^* = 6$, with boundaries at $15.1$, $17.6$,
$24.4$, $27.2$, and $38.0$ ms, clustered in the high-variance impact region.
Figure~\ref{fig:mcycle_partition} shows the K-selection curve and the
resulting partition.
With $n = 133$ observations, $K^* = 6$ implies bins of roughly
$10$--$30$ observations; the narrower bins in the chaotic $15$--$30$ ms
window improve within-bin exchangeability at the cost of fewer ECDF atoms.
This is the small-sample manifestation of the bias--variance trade-off
discussed in Section~\ref{sec:bias-variance}: the CV criterion selects
the best available tradeoff, not a guaranteed minimum bin size, and the
coarseness of the within-bin ECDF is absorbed into interval width rather
than miscoverage.

Figure~\ref{fig:mcycle_intervals} and Table~\ref{tab:mcycle_coverage}
compare prediction intervals averaged over $R=200$ random 50/50 splits.
In the matched-sample comparison (top block), Gaussian split conformal is
drastically inefficient ($172.4$ g), CQR (cubic) and CQR-IDR are moderately
better ($134.1$ g and $127.6$ g respectively), and CQR-QRF adapts most
flexibly ($87.9$ g).
Our method on the same training half (italic row, $n_{\rm tr} = 66$)
yields mean width $100.6$ g with sub-nominal coverage ($86.9 \pm 0.5\%$);
at this small sample size, CQR-QRF achieves narrower intervals, reflecting
the bias--variance trade-off discussed in Section~\ref{sec:bias-variance}:
fewer observations per bin produce a coarser within-bin ECDF and wider intervals.
When all $n$ observations are used (bold row, below the line), our method
achieves $90.3 \pm 0.2\%$ coverage with mean width $77.3$ g, narrower than
all competitors.
All split-conformal methods exceed nominal coverage ($92.5$--$93.1\%$).

\begin{figure}[ht]
\centering
\includegraphics[width=\textwidth]{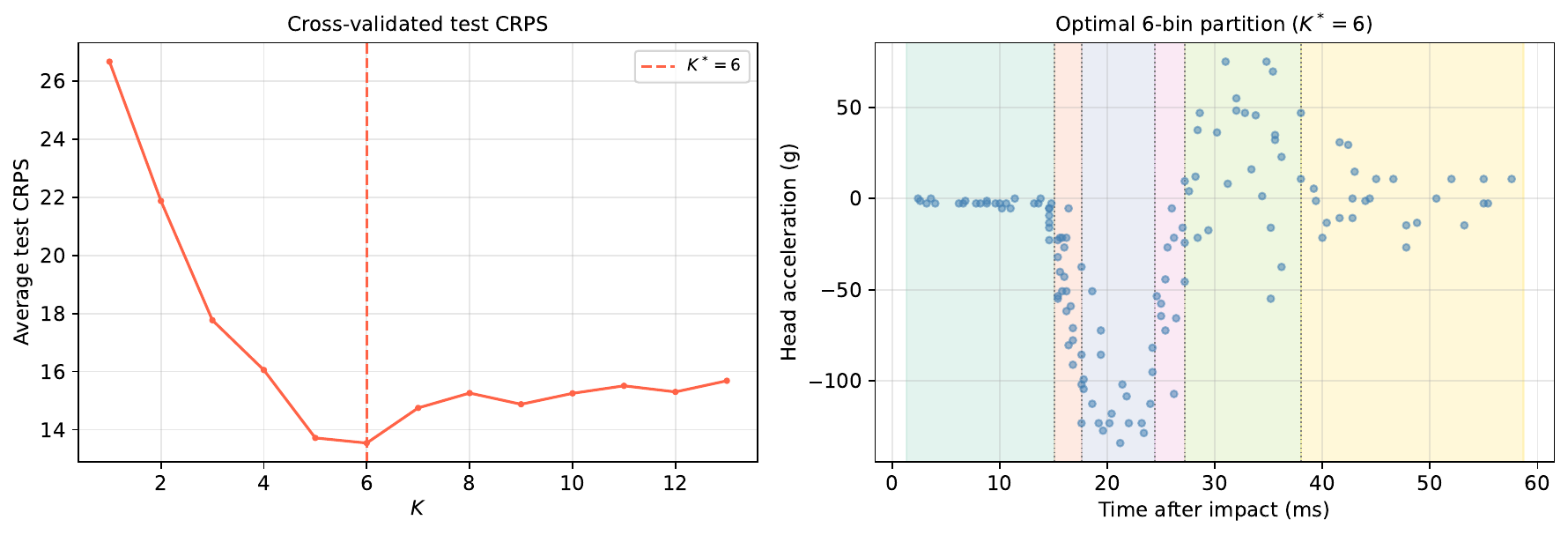}
\caption{Motorcycle accident. \emph{Left:} cross-validated test CRPS as a
  function of $K$ with the selected $K^* = 6$.
  \emph{Right:} scatter with the 6-bin partition; boundaries are
  concentrated in the high-variance impact phase.}
\label{fig:mcycle_partition}
\end{figure}

\begin{figure}[ht]
\centering
\includegraphics[width=\textwidth]{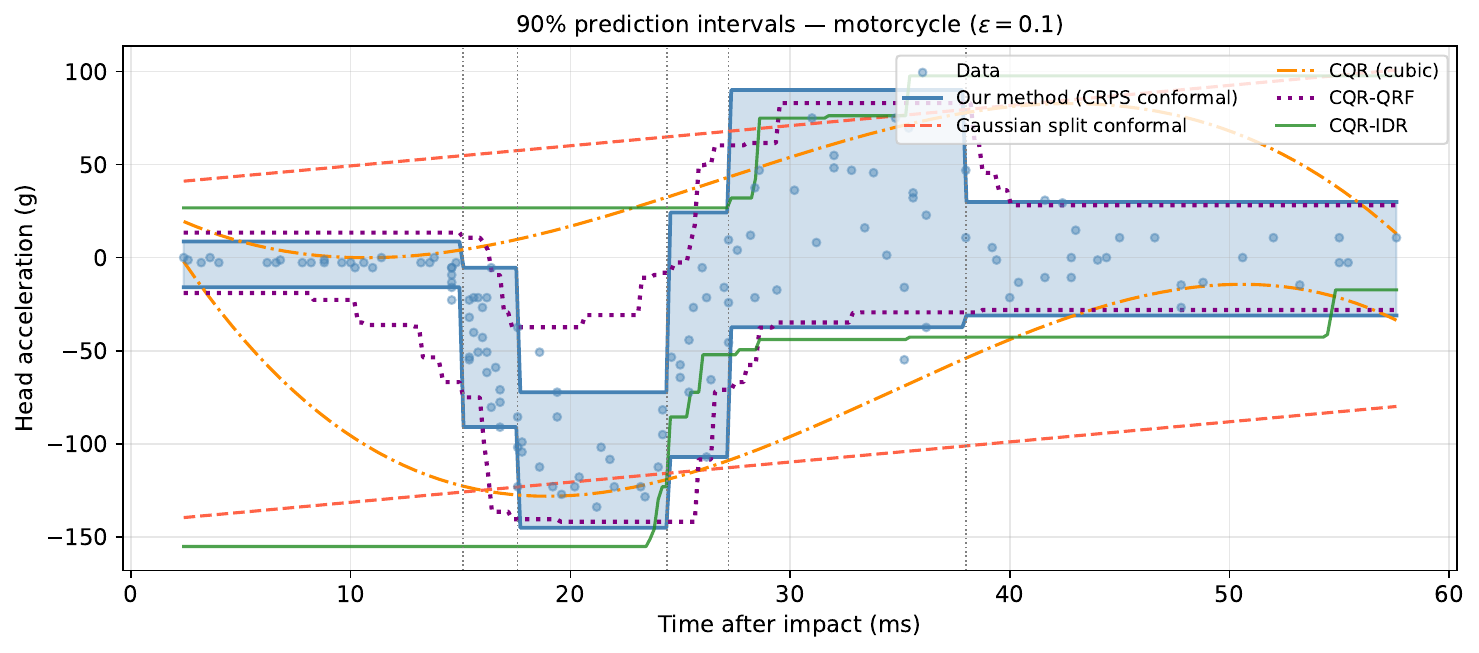}
\caption{Motorcycle accident: 90\% prediction intervals.
  Gaussian split conformal is constrained to constant width; our method,
  CQR-QRF, and CQR-IDR adapt to the non-stationary variance structure.
  A small number of data points in the high-variance impact phase ($20$--$30$ ms)
  fall outside the proposed method's intervals, consistent with the
  approximate within-bin exchangeability in this small and strongly heteroscedastic dataset.}
\label{fig:mcycle_intervals}
\end{figure}

\begin{table}[ht]
\centering
\begin{tabular}{lcc}
\hline
Method & Coverage (\%) & Mean width (g) \\
\hline
\textit{Our method ($n/2$)} & $86.9 \pm 0.5$ & $100.6 \pm 1.3$ \\
Gaussian split conformal & $92.5 \pm 0.0$ & $172.4 \pm 1.0$ \\
CQR (cubic) & $92.5 \pm 0.0$ & $134.1 \pm 1.5$ \\
CQR-QRF & $93.1 \pm 0.1$ & $87.9 \pm 0.7$ \\
CQR-IDR & $92.7 \pm 0.0$ & $127.6 \pm 0.6$ \\
\hline
\textbf{Our method (full $n$)} & $90.3 \pm 0.2$ & $77.3 \pm 0.2$ \\
\hline
\end{tabular}
\caption{Motorcycle accident ($n=133$): empirical coverage and mean width of
  prediction intervals at nominal level $1-\varepsilon = 0.90$,
  averaged over $R=200$ random 50/50 splits ($\pm$ one standard error).
  \emph{Top block}: all methods use the same training half ($n/2$);
  \textit{Our method ($n/2$)} is directly comparable to the competitors.
  \emph{Below the line}: \textbf{Our method (full $n$)} uses all $n$
  observations, a design advantage inherent to full-data conformal methods.}
\label{tab:mcycle_coverage}
\end{table}

\section{Additional considerations}
\label{sec:additional}
\paragraph{Limitations.}
In terms of scalability, it is important to note that the DP binning step has 
$O(n^2 K)$ time complexity and $O(n^2)$ storage requirement, which may be 
prohibitive for large datasets.
The method applies only to the case of a one-dimensional covariate and a contiguous binning structure.
In small samples the CV criterion can select a large $K$, resulting in bins with
few observations and hence coarse prediction intervals; this is a feature of the
bias-variance tradeoff under limited data, not a failure of validity.

\paragraph{Multi-dimensional extension.}
The LOO-CRPS cost $W(S)\cdot m/(m-1)^2$ can in principle be extended to multivariate data, but DP tractability
relies on the total order induced by sorting a one-dimensional $x$.
One could think of two natural extensions for $d > 1$.

The first is projection onto a learned one-dimensional index $g:\mathbb{R}^d\to\mathbb{R}$,
sorting by $g(x)$ and running the existing DP.
The partition is globally optimal \emph{given} $g$; alternating optimisation
over $(g, \text{partition})$ is tractable, but global optimality in $g$ is not
guaranteed.
An open question is whether the data-adaptive $g$ minimising LOO-CRPS recovers
a form of sufficient dimension reduction for the conditional \emph{distribution}
$P(Y\mid X=x)$, rather than merely the conditional mean.

The second is CART-style greedy binary splitting, requiring $O(dn)$ evaluations
per split and $O(Kdn)$ total.
No projection is needed and regions are interpretable rectangular cells, but global
optimality is again not guaranteed.

\paragraph{Open problems.}
At least three questions remain unresolved.
First, whether CRPS coherence (same rule for binning and scoring) can be proved to
yield tighter coverage bounds or more efficient prediction sets relative to a
mismatched pair.
Second, the monotone-coverage question: as $K$ increases, conditional coverage at
a fixed $x^*$ should improve; making this precise requires bounding the
exchangeability violation as a function of bin width and local DGP smoothness.
Third, the consistency of $K^*$: does
$K^*=\operatorname*{arg\,min}_{K}\overline{\mathrm{TestCRPS}}(K)$
converge to a meaningful oracle as $n\to\infty$, and under what conditions does
the population $\overline{\mathrm{TestCRPS}}(K)$ have a unique minimum?

\section{Conclusion}
\label{sec:conclusion}

We have presented a method for non-parametric conditional distribution estimation
that combines three independently motivated ideas into a coherent pipeline: optimal
bin-boundary placement by LOO-CRPS minimisation, cross-validated bin-count selection,
and conformal prediction using the within-bin ECDF as both the predictive distribution
and the nonconformity score.

The central technical contribution is the closed-form cost
$\mathrm{cost}(S) = mW/(m-1)^2$, which reduces the LOO-CRPS of any bin to a single
pairwise-dispersion scalar and enables exact globally optimal partitioning in
$O(n^2 K)$ time via dynamic programming.

On the predictive side, convexity of the CRPS nonconformity score guarantees
contiguous prediction intervals in the split-conformal case; in our transductive
setting, single-interval structure is observed in all experiments.
The $k$-NN score provides a bandwidth-free alternative that yields
non-contiguous highest-density regions for multimodal within-bin distributions.
The two scores are therefore complementary: CRPS for efficiency in unimodal regimes,
$k$-NN for distributional fidelity when the within-bin response is multimodal.

On the real datasets, the full-$n$ method (which uses all observations, unlike
competitors that reserve half for calibration) is competitive or superior in interval
efficiency: $11$--$40\%$ narrower intervals on Old Faithful (bimodal conditional) and
$2.2\times$ narrower than Gaussian split conformal on the motorcycle benchmark
(strongly heteroscedastic), while maintaining near-nominal coverage in both cases.
A matched-sample comparison (restricting our method to the same training half as the
competitors) shows the cost of halving the data: on Old Faithful the method remains
narrower than all competitors; on the motorcycle dataset ($n/2=66$, coverage
$86.9 \pm 0.5\%$) CQR-QRF achieves modestly narrower intervals ($87.9$ vs $100.5$ g),
reflecting the small-sample bias-variance regime.

\section*{Acknowledgements}

The author acknowledges Prof. Alexander Gammerman for his helpful guidance and Matteo Fontana 
for providing pointers to relevant prior work. 
The author used Claude (Anthropic) as a writing and coding assistant during the
preparation of this manuscript.
Claude assisted with drafting and editing prose, implementing experiment scripts,
and making \LaTeX{} edits.
The method, proofs, experimental design, and all intellectual content are the
author's own.

\section*{Software and Reproducibility}

The method is implemented in the \texttt{crpsconfreg} Python package,
available at \url{https://pypi.org/project/crpsconfreg}.
All figures and numerical results in this paper can be reproduced using the
scripts in the accompanying repository at \url{https://github.com/ptocca/crps-conformal-regression}.
An interactive browser demo (no installation required) is hosted at
\url{https://ptocca.github.io/crps-conformal-regression/}.

\bibliography{../references}

\appendix

\section{Proofs}
\label{app:proofs}

\subsection*{Proof of Proposition~\ref{prop:cost}}

Sum the per-observation expression over $k \in S$.
For the first term,
\[
  \sum_{k=1}^m \frac{d_k}{m-1}
  = \frac{1}{m-1}\sum_{k=1}^m d_k
  = \frac{D}{m-1},
\]
since $\sum_k d_k = \sum_k \sum_{\ell\ne k}|y_\ell - y_k| = D$.
For the second term, note that removing $k$ from the pairwise sum gives
$\sum_{\ell\ne r,\,\ell,r\ne k}|y_\ell-y_r| = D - 2d_k$, so
\[
  \sum_{k=1}^m \frac{D - 2d_k}{2(m-1)^2}
  = \frac{mD - 2D}{2(m-1)^2}
  = \frac{D(m-2)}{2(m-1)^2}.
\]
Subtracting and simplifying,
\[
  \mathrm{cost}(S)
  = \frac{D}{m-1} - \frac{D(m-2)}{2(m-1)^2}
  = \frac{mD}{2(m-1)^2}
  = \frac{mW}{(m-1)^2}.
\]
\noindent\hfill$\square$

%

\section{Algorithm Pseudocode and Implementation Details}
\label{app:impl}
\paragraph {Pseudocode.}
The pseudocode in Algorithm~\ref{alg:dp} on page~\pageref{alg:dp} implements the DP algorithm for computing 
the CRPS-optimal $K$-partition.

\begin{algorithm2e}[htp]
\caption{CRPS-optimal $K$-partition via dynamic programming}
\label{alg:dp}
\KwIn{Sorted observations $(x_{(1)},y_1),\ldots,(x_{(n)},y_n)$; number of bins $K$}
\KwOut{Optimal bin boundaries $0 = b_0 < b_1 < \cdots < b_K = n$}
\BlankLine
\textbf{Phase 1: Precompute cost matrix}\;
\For{$i = 1$ \KwTo $n$}{
  Initialise Fenwick trees $T_{\mathrm{cnt}}$, $T_{\mathrm{sum}}$; running total $\Sigma \gets 0$; $W \gets 0$\;
  \For{$j = i$ \KwTo $n$}{
    Insert $y_j$ into $T_{\mathrm{cnt}}$ and $T_{\mathrm{sum}}$; update $\Sigma \gets \Sigma + y_j$\;
    $r \gets T_{\mathrm{cnt}}.\mathrm{prefixQuery}(y_j)$;\quad
    $S_{\le} \gets T_{\mathrm{sum}}.\mathrm{prefixQuery}(y_j)$;\quad
    $S_{>} \gets \Sigma - S_{\le}$\;
    $W \gets W + y_j \cdot r - S_{\le} + S_{>} - y_j \cdot (j - i + 1 - r)$\;
    $m \gets j - i + 1$\;
    $c[i][j] \gets \begin{cases} m\,W/(m-1)^2 & \text{if } m \ge 2 \\ +\infty & \text{if } m = 1 \end{cases}$\;
  }
}
\BlankLine
\textbf{Phase 2: Fill DP table}\;
\For{$j = 2$ \KwTo $n$}{
  $\mathrm{dp}[1][j] \gets c[1][j]$;\quad $\mathrm{split}[1][j] \gets 0$\;
}
\For{$k = 2$ \KwTo $K$}{
  \For{$j = k$ \KwTo $n$}{
    $\mathrm{dp}[k][j] \gets +\infty$\;
    \For{$i = k-1$ \KwTo $j-1$}{
      $v \gets \mathrm{dp}[k-1][i] + c[i+1][j]$\;
      \If{$v < \mathrm{dp}[k][j]$}{
        $\mathrm{dp}[k][j] \gets v$;\quad $\mathrm{split}[k][j] \gets i$\;
      }
    }
  }
}
\BlankLine
\textbf{Phase 3: Backtrack boundaries}\;
$b_K \gets n$;\quad $k \gets K$;\quad $j \gets n$\;
\While{$k \ge 1$}{
  $b_{k-1} \gets \mathrm{split}[k][j]$;\quad $j \gets b_{k-1}$;\quad $k \gets k - 1$\;
}
\KwRet{$(b_0, b_1, \ldots, b_K)$}\;
\end{algorithm2e}

\paragraph{Choice of data structure.}
Two Fenwick trees~\citep{Fenwick1994} suffice: one storing counts (to obtain $r$)
and one storing values (to obtain $S_{\le}$).
A sorted array answers prefix queries in $O(1)$ but requires $O(m)$ per insertion,
raising total precomputation to $O(n^3)$.
A balanced BST achieves the same $O(\log n)$ per operation with higher constant
factors; the Fenwick tree is therefore the natural choice: simple, cache-friendly,
and sufficient for the $O(n^2\log n)$ bound.

\section{Tightness of \texorpdfstring{$O(n^2 K)$}{O(n\textasciicircum{}2 K)}: Quadrangle Inequality Counterexample}
\label{app:qi}

If $c(i,j)$ satisfied the \emph{quadrangle inequality} (QI),
$c(a,c)+c(b,d)\le c(a,d)+c(b,c)$ for all $a\le b\le c\le d$,
then by the Knuth--Yao theorem~\citep{Knuth1971,Yao1980} optimal split points
would be monotone in $j$, enabling $O(n)$ per DP row and total $O(nK)$.
However, the LOO-CRPS cost violates the QI already at $n=4$:
for $\mathbf{y}=(0,0,0,0,0,1,1,1,1,1)$ and $(a,b,c,d)=(1,4,5,6)$
the inequality fails by a gap of $0.30$.
One could view this as a consequence of the $m/(m-1)^2$ prefactor which makes the 
cost disproportionately sensitive to small bins.

\section{Venn Prediction Band}
\label{app:venn}

For test point $y^*$ in bin $B_k$ with $m$ training responses $y_1,\ldots,y_m$,
the Venn prediction is the family of augmented ECDFs
\[
  F_h(t) = \frac{1}{m+1}\Bigl(\#\{i : y_i \le t\} + \mathbf{1}[y_h \le t]\Bigr),
  \qquad y_h \in \mathbb{R}.
\]
Each $F_h$ is a valid CDF; the family spans all distributions consistent with one
additional observation in the bin.
At each $t$, the band covers
\[
  \underline{F}(t) = \frac{\#\{i : y_i \le t\}}{m+1},
  \qquad
  \overline{F}(t) = \underline{F}(t) + \frac{1}{m+1},
\]
a constant width of $1/(m+1)$ everywhere.
In terms of the training ECDF $\hat{F}_m$:
$\underline{F}(t) = \tfrac{m}{m+1}\hat{F}_m(t)$,\;
$\overline{F}(t) = \underline{F}(t) + \tfrac{1}{m+1}$.
Under exchangeability, this construction gives a valid multiprobabilistic prediction
in the sense of Vovk et al.

Figure~\ref{fig:venn} shows the band for each bin of the running example.
The width $1/(m+1)$ is determined entirely by bin size and carries no information
about the local density or the position of $y^*$ within the bin.

\begin{figure}[h]
\centering
\includegraphics[width=\textwidth]{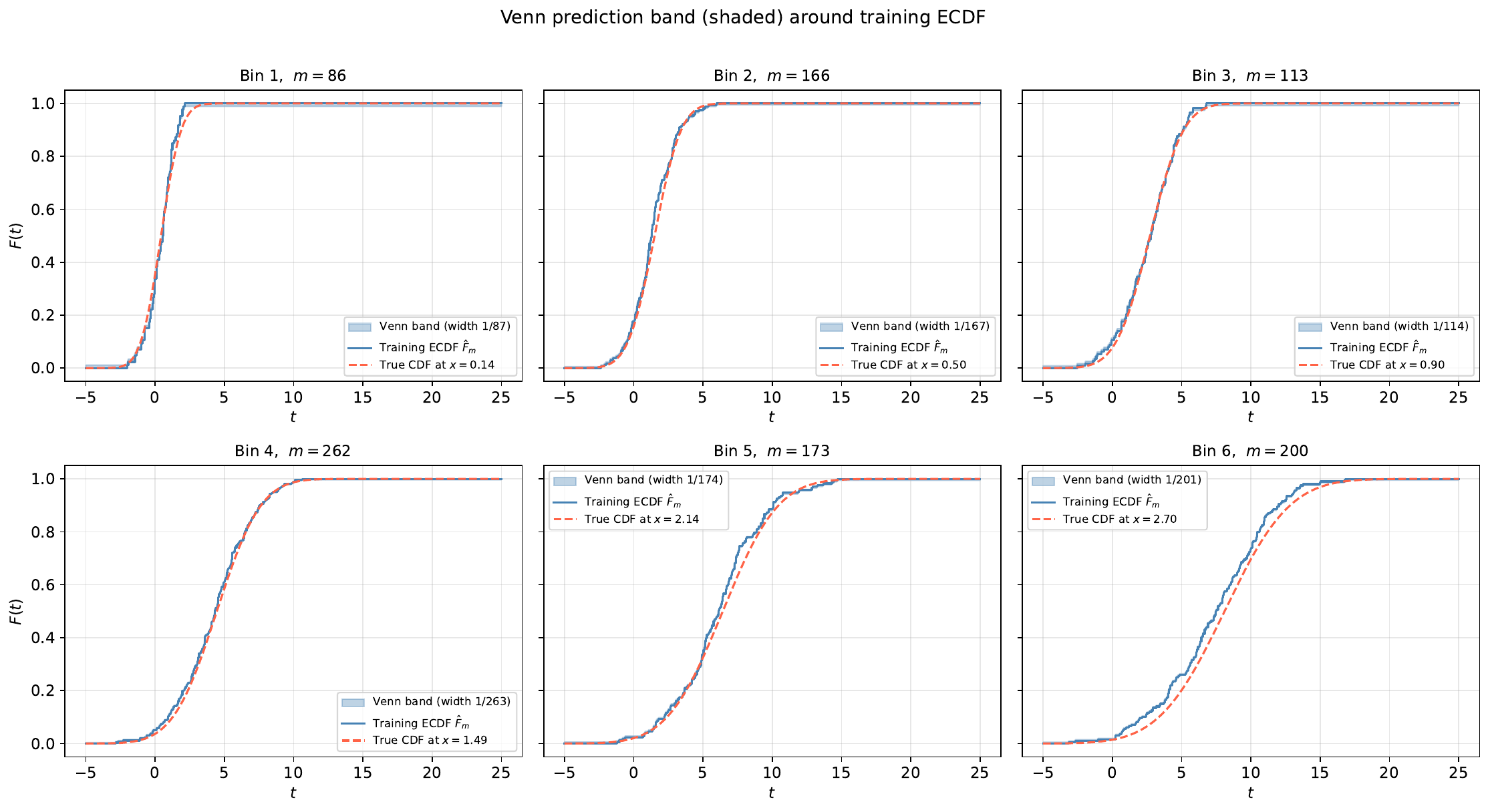}
\caption{Venn prediction band (shaded) and training ECDF $\hat{F}_m$ (step function)
for each of the six bins, alongside the true conditional CDF at the bin midpoint (dashed red).
The band width $1/(m+1)$ ranges from $0.004$ to $0.012$, invisible at these bin sizes.}
\label{fig:venn}
\end{figure}

\section{Exchangeability Analysis and Bias--Variance Trade-off}
\label{app:exch}

\subsection*{Exchangeability analysis}

Conformal validity (Proposition~\ref{prop:coverage}) requires two conditions:
\begin{enumerate}
  \item \textbf{Score symmetry.}
    For any fixed set of $m+1$ values in a bin, any permutation of
    $(y_1,\ldots,y_m,y^*)$ produces the same multiset of LOO nonconformity scores.
    This is a property of the LOO-CRPS computation and always holds.
  \item \textbf{Statistical exchangeability of bin members.}
    Under i.i.d.\ sampling, the $m+1$ values in the bin must be exchangeable
    as random variables, which requires that bin membership does not introduce
    dependence among the $y$-values.
\end{enumerate}

Condition~(2) is mildly violated by a data-dependent partition.
The DP uses all $y$-values to set boundaries, so conditioning on a bin's membership
acts as a $y$-dependent filter: observations ended up together partly
\emph{because} their $y$-values were compatible with the DP's cost criterion.

A concrete thought experiment makes this vivid.
Consider an observation near a bin boundary.
If its $y$-value were very different from its bin-mates, the DP might have placed
the boundary on the other side.
Conditioning on it being in \emph{this} bin therefore biases its $y$-distribution
toward values compatible with the present companions.

We argue informally that the violation is mild in practice: it is not in the score computation
(permutation-invariant for any fixed bin) but in the probability model over which
observations end up together.
For the partition to change, a single observation must alter the DP cost matrix
enough to shift a boundary --- increasingly unlikely as bin sizes grow.
\citet{AllenEtAl2025} provide formal support: auto-calibrated
within-bin predictors inherit conformal guarantees even for data-dependent partitions.

\subsection*{Bias--variance trade-off in $K$}

Wide bins (small $K$) average over a heterogeneous $x$-range, violating
exchangeability; the resulting conformal intervals may be mis-sized at a
specific $x^*$ --- conservative where local spread is below the bin average,
anti-conservative where it is above.

Narrow bins (large $K$) improve exchangeability but at two distinct costs.
First, the conformal p-value lies on the grid $\{1/(m+1),\ldots,1\}$,
so for $m < \lceil 1/\varepsilon\rceil - 1$ no point can be excluded at level
$\varepsilon$ and the prediction set is the whole real line.
At $\varepsilon=0.10$ this floor falls at $m < 9$.
Second, with small $m$ the conformal quantile is estimated from few scores,
so interval widths are highly variable across test points.

The CV criterion penalises both failure modes:
over-split partitions incur high test CRPS because the within-bin ECDF has too few
atoms, while under-split partitions incur high test CRPS because the ECDF averages
over too heterogeneous a region.
The U-shaped $\overline{\mathrm{TestCRPS}}(K)$ is therefore a genuine empirical optimum.

\section{Numerical Illustration}
\label{app:example}

We illustrate further the running example introduced in the main text in Section~\ref{sec:K}:
$n = 1000$,
$Y \mid X = x \sim \mathcal{N}(3x,(1+x)^2)$, $X \sim \mathrm{Uniform}(0,3)$.

\paragraph{Cross-validated $K$ selection.}
The 5-fold CV procedure selects $K^* = 6$; the CV curves are shown in
Figure~\ref{fig:kselect} and the resulting partition in Figure~\ref{fig:partition}.
The bin boundaries and sizes are listed below.

\paragraph{Conformal prediction intervals.}
Table~\ref{tab:ci} reports the $90\%$ conformal prediction intervals
($\varepsilon = 0.10$) at three representative test points, alongside the true $90\%$
intervals under the generating distribution.

\begin{table}[h]
\centering
\begin{tabular}{ccccc}
\hline
$x^*$ & Bin & $\Gamma^{0.1}$ & Width & True $90\%$ width \\
\hline
$0.3$ & 2 & $[-1.30,\; 3.98]$ & $5.28$ & $4.28$ \\
$1.5$ & 4 & $[-0.15,\; 8.74]$ & $8.89$ & $8.22$ \\
$2.7$ & 6 & $[\phantom{-}1.20,\; 13.48]$ & $12.27$ & $12.17$ \\
\hline
\end{tabular}
\caption{Conformal $90\%$ prediction intervals at three representative $x$-values.
Widths grow monotonically with $x$, closely tracking the oracle under the true distribution.}
\label{tab:ci}
\end{table}

Intervals are slightly conservative at $x^* = 0.3$ and $1.5$, reflecting residual
within-bin heterogeneity, and nearly exact at $x^* = 2.7$ (width $12.27$ vs $12.17$).

\paragraph{Empirical coverage.}
On a fresh test set of $2{,}000$ observations, empirical coverage is $94.6\%$,
$89.1\%$, and $79.7\%$ at $\varepsilon = 0.05$, $0.10$, and $0.20$ respectively,
all within the nominal guarantees' sampling uncertainty.

\paragraph{Bin boundaries and sizes.}
The $K^* = 6$ bins selected by cross-validation on the running example are:
\begin{align*}
  B_1 &: x \in [0.00,\,0.28],\;\; m_1 = 86, &
  B_2 &: x \in [0.28,\,0.72],\;\; m_2 = 166, \\
  B_3 &: x \in [0.72,\,1.09],\;\; m_3 = 113, &
  B_4 &: x \in [1.09,\,1.88],\;\; m_4 = 262, \\
  B_5 &: x \in [1.89,\,2.39],\;\; m_5 = 173, &
  B_6 &: x \in [2.40,\,3.00],\;\; m_6 = 200.
\end{align*}
The partition adapts to the jointly growing mean and variance:
the narrow bins $B_1$ and $B_3$ isolate transition regions where the mean slope
and rising $\sigma$ interact most strongly, while the wider bins are supported by
larger sample counts.

\paragraph{P-value distribution.}
Figure~\ref{fig:coverage} shows the distribution of conformal p-values $p(y^*)$ on the
$2{,}000$-point test set.
The near-uniform distribution is consistent with the coverage guarantee, with
slight conservatism reflecting the approximate exchangeability within bins of a
heteroscedastic process.

\begin{figure}[h]
\centering
\includegraphics[width=0.5\textwidth]{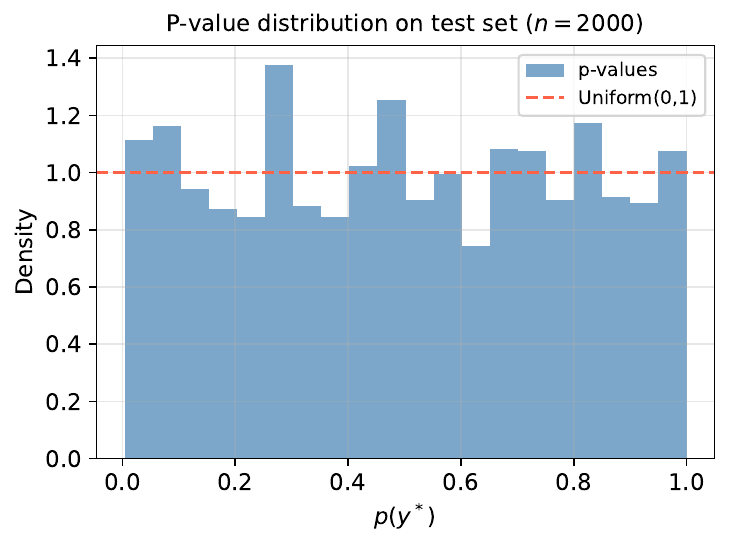}
\caption{Distribution of conformal p-values $p(y^*)$ on the $2{,}000$-point test set.
Under a valid conformal predictor the p-values are uniform, or super-uniform (density $\le 1$
everywhere); slight conservatism reflects the approximate exchangeability within bins.
P-values are discrete, taking values in $\{1/(m+1),\ldots,1\}$ for the $m$ training
points in each bin; the histogram is smoothed by the varying bin sizes across test
points.}
\label{fig:coverage}
\end{figure}

\paragraph{Conditional coverage.}
Figure~\ref{fig:condcov} reports empirical coverage per bin at three levels.
Bins~2--6 are close to the nominal level at all thresholds.
Bin~1 (the leftmost, $m=86$) under-covers: test points near the right edge
face a true conditional distribution shifted relative to the training ECDF,
an inherent price of binning a continuously varying process.

\begin{figure}[h]
\centering
\includegraphics[width=0.85\textwidth]{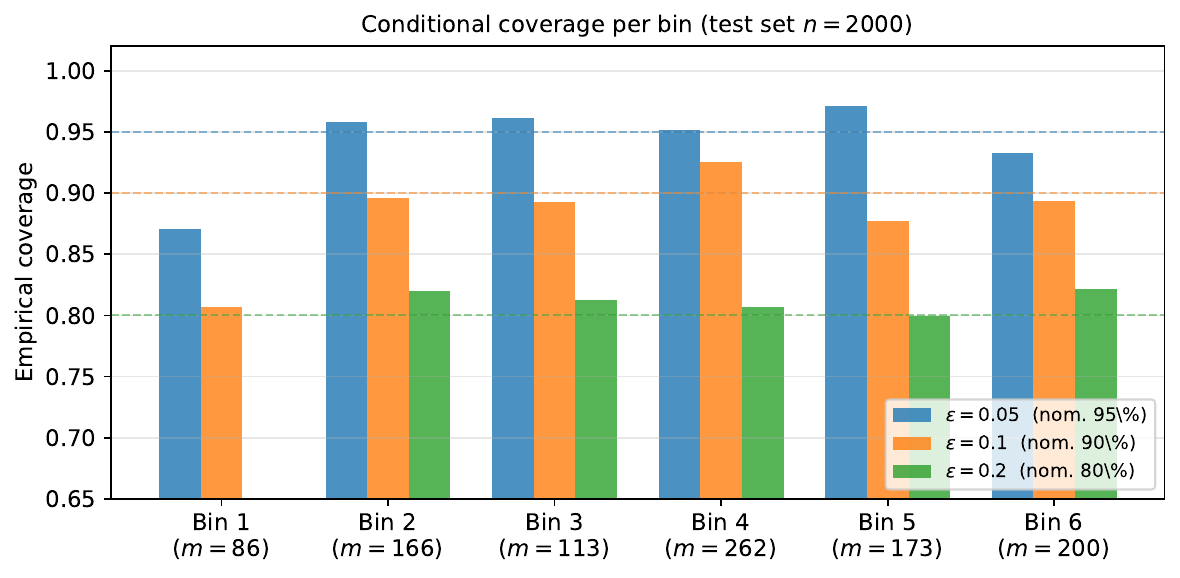}
\caption{Conditional coverage per bin on the $2{,}000$-point test set at three levels
$\varepsilon \in \{0.05, 0.10, 0.20\}$.
Dashed lines mark the nominal level.
Bin~1 ($m=86$, leftmost) under-covers due to within-bin shift of the true conditional
distribution; the remaining bins are close to their nominal targets.}
\label{fig:condcov}
\end{figure}

\paragraph{PIT histograms.}
Figure~\ref{fig:pit} shows the Probability Integral Transform histograms per bin.
Under perfect calibration the PIT values $\hat{F}_b(y^*)$ are uniform.
Bin~1 shows a U-shape (boundary effect); bins~5--6 show a mild rightward skew.
Interior bins are close to uniform, confirming that the CRPS-optimal partition
produces well-calibrated predictive distributions where the bin size is adequate.

\begin{figure}[h]
\centering
\includegraphics[width=\textwidth]{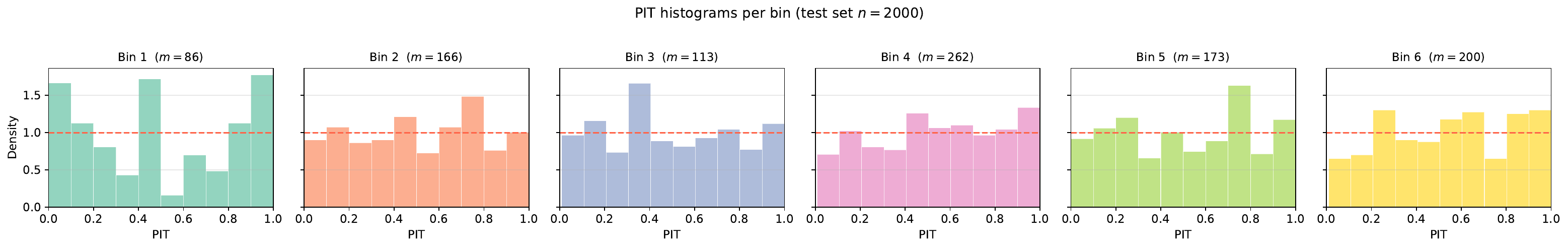}
\caption{PIT histograms per bin.  The PIT for test point $(x^*, y^*)$ assigned to
bin~$b$ is $\hat{F}_b(y^*) = m_b^{-1}\sum_{i=1}^{m_b}\mathbf{1}[y_{b,i} \le y^*]$.
Under perfect calibration the histogram is uniform (dashed line at density~1).
Boundary bins (1 and 6) show departures due to within-bin heterogeneity.}
\label{fig:pit}
\end{figure}

\section{Discussion: Scoring Rules and Cram\'{e}r Distance}
\label{app:scoring}

\paragraph{Other possible scoring rules.}
Among strictly proper rules for the full distribution, we are not aware of any
other whose LOO sum reduces to a scalar computable from pairwise differences:
the log score and Brier score evaluated on a fine grid would require $O(mn)$ or
$O(m\log m)$ work per bin extension rather than $O(\log n)$, compromising the
$O(n^2\log n)$ precomputation.
Whether the coherence between the LOO-CRPS binning criterion and the CRPS
nonconformity score yields formal benefits (tighter bounds, more efficient sets
relative to a mismatched pair) is an open question.

\paragraph{Cram\'{e}r distance equivalence.}
A natural variant of the cross-validated criterion forms the empirical CDF
$\hat{G}_n$ of all held-out observations in a bin and evaluates the predictive
CDF $F$ against it using the \emph{Cram\'{e}r distance}~\citep{Szekely2013,Baringhaus2004}
\[
  d_C(F,\hat{G}_n)
  = \int_{-\infty}^{\infty}\!\bigl(F(t) - \hat{G}_n(t)\bigr)^2\,\mathrm{d}t.
\]
When $\hat{G}_n = \delta_y$ this reduces to $\mathrm{CRPS}(F,y)$~\citep{GneitingRaftery2007}.
Expanding the square gives
\begin{equation}
  \label{eq:cramer-crps}
  d_C(F,\hat{G}_n)
  = \frac{1}{n}\sum_{i=1}^n \mathrm{CRPS}(F,y_i)
  + \int_{-\infty}^{\infty}\hat{G}_n(t)\bigl[\hat{G}_n(t)-1\bigr]\,\mathrm{d}t.
\end{equation}
The second term depends only on $\hat{G}_n$, so the correction cancels in every
pairwise comparison.
The two criteria are therefore identical for all purposes that matter: same
minimiser, same ranking of candidate partitions, same variance of the
model-selection statistic.

\end{document}